\documentclass[pdflatex,sn-mathphys-num]{sn-jnl}

\usepackage{graphicx}%
\usepackage{multirow}%
\usepackage{amsmath,amssymb,amsfonts}%
\usepackage{amsthm}%
\usepackage{mathrsfs}%
\usepackage[title]{appendix}%
\usepackage{xcolor}%
\usepackage{textcomp}%
\usepackage{manyfoot}%
\usepackage{booktabs}%
\usepackage{algorithm}%
\usepackage{algpseudocode}%
\usepackage{listings}%
\usepackage{enumitem}

\usepackage{tikz}
\usetikzlibrary{positioning,arrows.meta}

\newif\ifshowrevisions
\showrevisionstrue
\definecolor{revisionblue}{RGB}{0,70,150}

\theoremstyle{thmstyleone}%
\newtheorem{theorem}{Theorem}
\newtheorem{corollary}[theorem]{Corollary}%

\theoremstyle{thmstyletwo}%

\theoremstyle{thmstylethree}%

\begin{document}

\title[Article Title]{Time series causal discovery with variable lags}

\author*[1]{\fnm{Bruno} \sur{Petrungaro}}\email{b.petrungaro@qmul.ac.uk}

\author[1]{\fnm{Anthony C.} \sur{Constantinou}}\email{a.constantinou@qmul.ac.uk}

\affil*[1]{\orgdiv{Bayesian Artificial Intelligence research lab, MInDS research group}, \orgname{Queen Mary University of London}, \orgaddress{\street{Mile End Road}, \city{London}, \postcode{E1 4NS}, \country{UK}}}

\abstract{Causal Bayesian Networks (CBNs) are a powerful tool for reasoning under uncertainty about complex real-world problems. Such problems evolve over time, responding to external shocks as they occur. To support decision-making, CBNs require a cause-and-effect map of the variables under consideration, known as the network's structure. Learning the graphical structure of a causal model from data remains challenging; learning it from time-series data is even harder because dependencies may arise at different time lags. We construct a method that jointly learns a temporal ordering and a dominant lag for each selected ordered variable pair. Our contribution is a decomposable BIC-style objective score for mixed data that combines node-specific sample sizes with a penalty favouring recent lags. We optimise this objective using a Tabu search augmented with change-lag moves and greedy per-child lag tuning, while its node-wise decomposition supports cached and parallel neighbourhood evaluation. We establish acyclicity and complexity results and derive recovery limits when effects are distributed across multiple lags. In simulations, we show that the proposed method outperforms an inflexible lag-1 score-based learner in most evaluated settings, whereas most specialised methods achieve higher recovery accuracy on data that satisfy their assumptions. On a real-world UK COVID-19 dataset, the method learnt an expected short-lag-skewed structure; nevertheless, 46.8\% of selected edges had lags greater than one, showing that longer delays can be retained when their improved fit offsets the penalty. Under its no-contemporaneous-edges and causal-sufficiency assumptions, the method provides a computationally tractable subroutine for edge-specific single-lag optimisation, complementing score-based causal discovery methods.
}

\maketitle

\section{Introduction}\label{intro}

Bayesian networks (BNs) are a class of probabilistic graphical models (PGMs) introduced by Pearl (\cite{bib1, bib2}). BNs represent variables as nodes and their conditional dependencies as edges in a directed acyclic graph (DAG). Assuming that edges encode causal dependencies, a BN can be interpreted as a Causal Bayesian Network (CBN), in which the DAG captures cause-and-effect relationships among the variables. This makes CBNs powerful and an inherently explainable tool for modelling real-world systems, as they can simulate the effects of interventions. 

Pearl and Mackenzie (\cite{bib3}) describe the reasoning capabilities of models in terms of a ``ladder of causation". Rung one represents association, rung two represents intervention, and rung three represents counterfactual reasoning. Using association, which most AI models today provide, we can answer questions about how seeing one thing changes our belief about seeing another. This is probably the most common use of AI: prediction. In rung two, we can determine what happens to something when another element in the system changes. At rung three, we can explore what would have happened if we had taken a different approach. Unlike most traditional machine learning methods, CBNs traverse all the rungs of the ladder, making them well suited to support decision-making.

However, identifying a causal structure from observational data remains a major challenge. A specific aspect of learning from time series is how to model variable lags across different pairs of variables. This could be useful in real-world systems that exhibit dependencies across multiple time scales, arising from various sources of information. For example, in healthcare, physiological processes may exhibit different temporal dynamics than behavioural processes. Behavioural changes in a patient typically occur over longer periods. Learning structures with variable lags would enable us to capture these temporal dynamics, where some variables influence future outcomes almost immediately, while others have longer time-lagged effects.

Verma and Pearl (\cite{bib4}) describe how more than one BN, with a similar but different structure, can generate the same joint probability distribution. This set of equivalent structures, a Markov Equivalent Class, is often called a Completed Partially Directed Acyclic Graph (CPDAG) and contains both directed and undirected edges. A directed edge will appear in the CPDAG if it has the same orientation in all DAGs in the equivalence class. CPDAGs occur because, from observational data alone, only some edge orientations are identifiable.

In this paper, the time-series structural setting involves copying each variable at each time step, and then drawing directed edges between time-stamped copies, consistent with temporal order. That is, the time-unrolled representation only allows edges of the form $v_{i,t-\ell}\rightarrow v_{j,t}$ with $\ell\ge1$, where it is possible that $i=j$. Therefore, this time-unrolled graph is acyclic because time strictly increases along every edge, and no contemporaneous (i.e., same-time) edges are allowed. Furthermore, two admissible time-unrolled DAGs with the same skeleton must be identical, and each Markov-equivalence class contains at most one graph satisfying the temporal restrictions. Following from this, we assess recovery of the true DAG. A compact lagged graph of the time-unrolled graph is the one-node-per-variable summary, where edges are annotated with lags. Note that this corresponding compact lagged graph may contain directed feedback cycles (e.g., $X\rightarrow Y$ and $Y\rightarrow X$ at different lags), which is standard in dynamic models, but which do not violate acyclicity when unrolled as a DAG structure with time-lagged variables. We use a scoring function in Sec.~\ref{sec:score} that augments a decomposable BIC score with lag-dependent penalties and node-specific effective sample sizes.

We study score-based causal structure learning from multivariate mixed-data time series where parent–child effects may occur at different delays. We propose a Tabu-based structure learning algorithm that extends the search to include lag-adjustment moves and uses a decomposable mixed-data objective score that combines a BIC term with node-specific effective sample sizes and an explicit lag-length penalty. We establish theoretical properties for the algorithm, describe a parallel implementation to improve scalability, and evaluate recovery of both adjacencies and lags in simulation. Comparing the proposed method to established alternatives. We further illustrate the method on a real-world UK COVID-19 policy dataset, where it identifies plausible delayed relationships.

\section{Background}\label{background}

Let $V = (v_1,\dots,v_N) $ be an arbitrary ordering of the variables; then the BN $\mathcal{B}$ over $V$ is a pair $(\mathcal{B}_S,\mathcal{B}_P)$. $\mathcal{B}_S$ is a DAG, where each variable $v \in V$ is represented by a node. $\mathcal{B}_S$ is also called the network structure, and this study concerns algorithms that learn this structure from data. This is a challenging NP-Hard problem (\cite{bib5}). $\mathcal{B}_P$ is a set of functions, one $\forall\, v \in V$, defining a conditional probability distribution of the variable given $Pa(v)$, where $Pa(v)$ denotes the parents of variable $v \in V$. These functions quantify the probabilistic dependency strength between each variable connected by an edge in $\mathcal{B}_S$.

BNs obey the Local Markov property, which states that a child is conditionally independent of all its non-descendant nodes given its parent nodes. Non-descendants of $v$ are the variables $v' \in V$ with $v' \neq v$ such that $v'$ is neither a child nor a descendant of $v$ (in $\mathcal{B}_S$). In simple terms, these variables do not come after $v$ in a chain-variable dependency. This property links the structure of the network to the distribution it represents, implying that the standard chain rule for expressing a joint probability distribution:

\begin{equation}
    P(v_1,...,v_N) = \prod_{i=1}^N P(v_i \mid v_1,...,v_{i-1})
\end{equation}

can be expressed more compactly as:

\begin{equation}
    P(v_1,...,v_N) = \prod_{i=1}^N P(v_i \mid Pa(v_i))
\end{equation}

We can use this joint probability to support decision-making as we climb the ``ladder of causation". The conditional independence between variables implied by the Local Markov property means that knowing the value of certain variables makes some others irrelevant for predicting $v$. Let $\mathrm{ND}(v)$ denote the set of non-descendants of $v$, in the context of the Local Markov property,

\begin{equation}
    P(v\mid Pa(v),ND(v))=P(v \mid Pa(v))
\end{equation}

The local Markov property yields the convenient $d$-separation criterion, which graphically determines whether two variables are conditionally independent given any other set of variables. 

Most structure learning algorithms make assumptions that are often quite restrictive for modelling real-world data. These include the faithfulness assumption, which states that there are no independence relationships in the data that are not implied by the graph's structure, as well as the causal sufficiency assumption, which assumes that there are no unobserved latent confounders. Other common assumptions about the input data include the absence of missing values and distributional assumptions, for example, normality for continuous variables.

\subsection{Structure Learning Algorithms}\label{2.1}

Structure learning algorithms fall into several categories. First, score-based methods fall under the classical machine learning class of learning, in which different graphs are searched, and each visited graph is scored. Many scoring functions are available, and most of these check whether the model’s predicted pattern looks like the pattern observed in the data. Probably the most commonly used score in practice is the BIC score (\cite{bib6}), which maximises the likelihood while penalising model complexity. Score-based methods return the highest-scoring graph they discover. These methods are generally based on greedy search, as the number of possible structures is superexponential in the number of variables, and exhaustive search is infeasible. However, it is documented that greedy search can be asymptotically reliable under specific conditions (\cite{bib7}).

The second class of structure learning algorithms is constraint-based methods. These methods produce graphs that satisfy a set of conditional independence statements by using the properties of BNs to infer the graphical structure. Hybrid structure-learning algorithms are closely related to this class and to the previous one, as they rely on properties of both score-based and constraint-based methods, yet are typically considered a separate class of algorithms.

\subsection{Structure Learning from Time-Series}\label{2.2}

Learning a causal structure from time-series data generally involves two additional challenges compared with traditional structure learning. Specifically:
\vspace{1em}
\begin{itemize}
    \item The edges between variables may have a time-lagged effect.
    \item The graphical structure itself may change over time, even between the same set of variables, to reflect causal changes in the underlying system over time.
\end{itemize}
\vspace{1em}
In this study, we focus on the first challenge: learning a single, stable, time-invariant directed structure while explicitly identifying edge-specific lags. We do not model structural change over time; methods that do are reviewed only for context. Some progress has been made in the field of structure learning from time-series data, as detailed below.

\subsubsection{Structural changes over time}

Kocacoban and Cussens (\cite{bib8}) introduce two online structure-learning algorithms (Online Fast Causal Inference (OFCI) and Fast Online Fast Causal Inference (FOFCI)) that relax the assumption that the structure remains static over time. OFCI is an online version of Fast Causal Inference (FCI) that handles latent variables. It works by revising correlations as new data points arrive, then relearning the structure. FOFCI modifies OFCI to accelerate learning by leveraging relationships learnt from previous models.

Kummerfeld and Danks (\cite{bib9}) introduce Dynamic Online Causal Learning (DOCL), a structure learning algorithm designed to handle structures that change unpredictably. The algorithm processes data online in real time, tracking changes in the causal structure and probabilistic relationships learnt from sequential or ordered data.

Sometimes the underlying data-generating process changes; hence, Kummerfeld and Danks (\cite{bib10}) developed the Local Stationarity Structure Tracking (LoSST) algorithm for structure learning. This algorithm can adapt to these changes, primarily when the data originated from processes that are only locally stationary, and can dynamically track structural and relational changes in real time.

\subsubsection{Lagged dependencies with a stable structure}

Runge et al. (\cite{bib11}) introduce the Peter-Clarke Momentary Conditional Independence (PCMCI) algorithm, which identifies causal relationships by iterative conditional independence testing. It accounts for nonlinear relationships and high-dimensional data. This is achieved through adaptive methods that minimise long runtimes, even in scenarios with numerous variables and time lags. The proposed algorithm also quantifies the strength of the causal relationships found.

Constraint-based time-series causal discovery often has low recall when series are autocorrelated, particularly in the presence of latent confounders. Gerhardus and Runge (\cite{bib12}) pinpoint a key reason: low effect sizes in conditional independence (CI) tests due to unfortunate conditioning sets. Therefore, they increase the CI test effect size by restricting conditioning sets and enriching them with known or inferred parents/ancestors of the tested variables. This makes true dependencies easier to detect. They also add new orientation rules that infer ancestry during edge removal, not only after adjacency discovery, and named the algorithm they built to do this LPCMCI, which is an extension of PCMCI.

Siracusa and Fisher (\cite{bib13}) employed Bayesian inference over graphical structures to describe relationships among multiple vector time series. It assumes a fixed dependence structure over time series by introducing a Bayesian framework to infer time-indexed graph structures from time-series data. Malinsky and Spirtes (\cite{bib14}) learn causal structure from multivariate time series when there may be latent confounders and contemporaneous influences, but no contemporaneous feedback cycles. They formalise the data as a Structural Vector Autoregression with latent components and stationarity. 

Pamfil et al. (\cite{bib15}) introduce DYNOTEARS, a score-based method extending continuous optimisation for acyclic structure learning to dynamic BNs. It jointly estimates contemporaneous and weighted-lagged adjacency matrices by minimising a penalised least-squares objective, subject to the constraint that the contemporaneous graph contains no directed cycles. The method assumes linear continuous relationships and specifies a common maximum-lag window, with sparsity penalties determining which coefficients within that window are retained.

Hyvärinen et al. (\cite{bib16}) propose VARLiNGAM, which combines vector autoregression with the Linear Non-Gaussian Acyclic Model. The autoregressive component captures lagged relationships, while the non-Gaussianity and mutual independence of the disturbances identify the contemporaneous causal structure. This approach is applicable to linear continuous time series and relies critically on non-Gaussian disturbances for causal identification.

More recently, Paldino and Bontempi (\cite{bib17}) introduce Temporal Dependency to Causality (TD2C), a supervised learning approach that represents each candidate temporal relationship using features derived from mutual and conditional mutual information, forward–reverse transfer-entropy asymmetries, residual correlations, higher-order moments and regression coefficients. A classifier trained on simulated multivariate time-series systems, whose causal structures are known by construction, learns asymmetries that distinguish causal from non-causal relationships. It is subsequently applied to previously unseen time series. TD2C therefore differs from constraint-based and score-based methods by treating causal discovery as a transfer learning problem.

\section{Proposed Variable-Lag Tabu Search for Time-Series Structure Learning}

Bouckaert first proposed tabu search for structural learning (\cite{bib18}). This section introduces our proposed extension of Tabu-based search to time-series structure learning with edge-specific lags. We build on Tabu search because it naturally supports a discrete move set (add/delete/reverse/change-lag) under acyclicity and temporal constraints, and it is better equipped to avoid local optima than Hill-Climbing (HC) search (\cite{bib18}). It can be paired with decomposable scores, which make local score updates computationally inexpensive: a node's score changes only when its parent set changes, so there is no need to recalculate the score for the entire graph. 

Tabu search operates directly over the permitted parent-lag assignments and can accommodate the proposed decomposable objective score function, temporal restrictions and lag-adjustment moves. Because temporal order already fixes the direction of every admissible time-unrolled adjacency, the score is not required to select among Markov-equivalent orientations. Instead, it ranks candidate structures that differ in edge inclusion, ordered variable pairs or assigned lags.

\subsection{Tabu Search}

Given a set of variables $X = \{x_1, \dots, x_n\}$ and a dataset $D$, our aim is to find a DAG $G^{*} \in \mathcal{G}$ that maximises the scoring function $S: \mathcal{G} \to \mathbb{R}$:

\begin{equation}
    G^{*} = \arg\max_{G \in \mathcal{G}} S(G)
\end{equation}

For a current DAG $G_t$, define its neighbourhood as:

\begin{equation}
    N(G_t) = \Big\{ G \in \mathcal{G} \,\Big|\, G \text{ is obtained from } G_t \text{ by a single edge modification and is acyclic} \Big\}
\end{equation}

These single-edge modifications are additions, deletions, or reversals of an edge. The pseudo-code for Tabu Search is shown below in Algorithm~\ref{alg:tabu}.

\begin{algorithm}[ht]
\caption{Tabu Search}
\label{alg:tabu}
\begin{algorithmic}[1]
\State \textbf{Input:} Dataset $D$, initial DAG $G_0$, scoring function $S$, tabu tenure $L$, maximum iterations $T_{\mathrm{maxit}}$.
\State Initialize: $G \leftarrow G_0$; best solution $G^{*} \leftarrow G_0$; tabu list $T \leftarrow \emptyset$; iteration counter $t \leftarrow 0$.
\While{\(t < T_{\mathrm{maxit}}\)}
    \State Compute the neighbourhood:
    \[
      N(G) = \Big\{ G' \,\Big|\, G' \text{ is obtained by one edge modification from } G \text{ and is acyclic} \Big\}.
    \]
    \State Form the candidate set:
    \[
      C = \Big\{ G' \in N(G) \,\Big|\, \text{either the move } (G \to G') \notin T \text{ or } S(G') > S(G^{*}) \Big\}.
    \]
    \If{\(C = \emptyset\)}
       \State \textbf{break} the loop.
    \EndIf
    \State Select:
    \[
      G' \leftarrow \arg\max_{H \in C} S(H).
    \]
    \If{\(S(G') > S(G^{*})\)}
       \State Update best solution: \(G^{*} \leftarrow G'\).
    \EndIf
    \State Update the tabu list \(T\):
          \begin{itemize}
              \item Add the reverse move \((G' \to G)\) with tenure \(L\).
              \item Decrease the tenure of all moves in \(T\); remove moves whose tenure reaches 0.
          \end{itemize}
    \State Set \(G \leftarrow G'\).
    \State Increment iteration counter: \(t \leftarrow t + 1\).
\EndWhile
\State \textbf{return} \(G^{*}\).
\end{algorithmic}
\end{algorithm}

\subsection{Variable-Lag Tabu Search for Time-Series Structure Learning}\label{sec:tabu-variable}

Recall that $V = (v_1,\dots, v_N)$ is an arbitrary ordering of the unlagged variables we are considering. Each variable is an unlagged target at time $t$; candidate parents are its lagged copies. Figure~\ref{fig:varlag} below illustrates this principle.

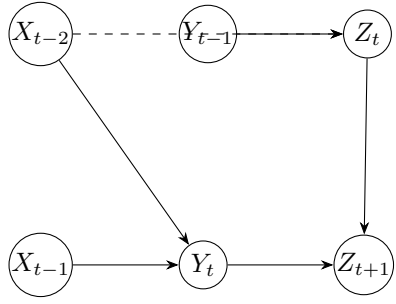
\begin{figure}[ht]
\centering
\begin{tikzpicture}[
  >=Stealth,
  node distance=2.2cm and 1.4cm,
  v/.style={circle,draw,minimum size=18pt,inner sep=0pt}
]
\node[v] (Xt2) {$X_{t-2}$};
\node[v, right=of Xt2] (Yt1) {$Y_{t-1}$};
\node[v, right=of Yt1] (Zt)  {$Z_t$};

\node[v, below=of Xt2] (Xt1) {$X_{t-1}$};
\node[v, right=of Xt1] (Yt)  {$Y_t$};
\node[v, right=of Yt]  (Ztp1){$Z_{t+1}$};

\draw[->]         (Xt1) -- (Yt);         
\draw[->]         (Xt2) -- (Yt);         
\draw[->]         (Yt1) -- (Zt);         
\draw[->, dashed] (Xt2) -- (Zt);         
\draw[->] (Zt) -- (Ztp1);   
\draw[->] (Yt) -- (Ztp1);   
\end{tikzpicture}
\caption{Graph with edge-specific lags. Solid arrows are selected; dashed is a candidate long-lag edge.}
\label{fig:varlag}
\end{figure}

We work on a time grid $\mathcal{T}=\{1,\dots,T\}$ and fix a maximum lag $L_{\max}\ge 1$. Let $X \;=\; \{\, v_{i,t} \;|\; i=1,\dots, N,\; t=1,\dots, T \,\}$ where each $v_{j,t}$ may have parents amongst lagged copies of any variable, restricted to the previous \(L_{\max}\) time steps:
\[
  Pa(v_{j,t}) \;\subseteq\; \{\, v_{i,t-\ell} \;|\; i=1,\dots,N,\; \ell=1,\dots,L_{\max},\; t-\ell\ge 1 \,\}.
\]

Therefore, edges are only allowed from $v_{i,s}$ to $v_{j,t}$ if $s < t$ and $t-s \leq L_{\max}$, where it is possible that $i=j$. Let \(\mathcal{B}_S=(X,E)\) be the DAG with edge set \(E \subseteq \{\, v_{i,t-\ell}\to v_{j,t}\,\}\) obeying these temporal constraints. We partition $V$ (hence each $v_{i,t}\in X$) into continuous and discrete types, allowing mixed data: if $v_i \in V^\text{disc}$, it takes values in a finite set $\{0,1,...,k\}$; if $v_i \in V^{\text{cont}}$, it takes values in $\mathbb{R}$. The dataset we work with is, therefore, $D = \{v_{i,t} \mid i = 1,..,N, t = 1,..., T\}$ where $v_{i,t}$ is the value of variable $i$ at time $t$. All variables in $V$ are observed time series; lagged variables are not separate inputs, but generated as time-shifted copies $v_{i, t - l}$ when constructing parent sets up to $L_{max}$.

We represent candidate structures in a compact form at the variable level as a set of lagged edges $E_\ell \subseteq \{(i\!\to\!j,\ell): i,j\in\{1,\dots, N\}, \ell\in\{1,\dots, L_{\max}\}\}$. This compact representation can contain cycles across variables. However, this compact form induces a time-unrolled DAG over $X=\{v_{i,t}\}$; i.e., by including edges $v_{i,t-\ell}\to v_{j,t}$ for all valid $t$. This is because $\ell\ge1$, so the time index increases along every directed edge.

For each ordered variable pair $(i,j)$, the compact representation permits at most one selected lag:

\[
\left| \{l:(i \rightarrow j, l) \in E_{l}\}\right| \le 1.
\]

Thus, different ordered pairs may have different delays, but a given parent–child relationship cannot operate simultaneously at several lags. This restriction is a deliberate sparsity assumption.

The single-delay restriction provides a parsimonious dominant-delay representation and is particularly useful when the time dimension is limited relative to the number of variables and candidate lags. In an unrestricted multi-lag representation, each outcome may have up to $N \times L_{max}$ lagged predictors, and a single ordered pair may contribute as many as $L_{max}$ coefficients. Our representation estimates at most one coefficient per ordered pair, with its lag selected during structure search. This improves the parameter-to-observation ratio, reduces the risk of overfitting and limits collinearity among neighbouring lagged copies of the same variables, which are likely to be autocorrelated. It also provides a straightforward interpretation of the selected lag as the dominant or best-supported delay for that relationship. However, when the true effect is distributed across several lags, the selected edge represents a parsimonious single-delay approximation rather than the complete temporal effect profile. These advantages are relevant to the COVID-19 application in Section~\ref{sec:covid}, which contains approximately 866 daily observations for 46 variables. With $L_{max} = 6$, an unrestricted lag expansion would generate as many as 276 candidate lagged predictors for each outcome, whereas the proposed representation permits at most 46 selected parent coefficients.

\subsubsection{Scoring function}\label{sec:score}

We extend the BIC decomposable score used in score-based BN structure learning (\cite{bib6, bib19, bib20, bib21}) to the time-unrolled setting, i.e., treating a time-series model as a sequence of time steps so that each variable gets a copy at each time step. We then add a decomposable lag-length regulariser that acts as a structural prior favouring shorter delays, analogous to \cite{bib22}, \cite{bib23}, and \cite{bib24}. Using $n_j \;=\; T - \max_{(i,\ell)\in Pa_\ell(v_j)} \ell$ where $Pa_\ell(v_{j,t}):= \{\, (i,\ell): v_{i,t-\ell} \in Pa(v_{j,t}) \,\}$ penalises models that sacrifice observations by requiring long lags, and the explicit lag penalty $-\lambda \sum_{(i,\ell)\in Pa_\ell(v_{j,t})} \max(0,\ell-1)$ regularises toward parsimonious models, keeping only long delays that materially improve fit. The penalty is imposed on each edge for every lag step beyond 1. If an edge has lag $\ell > 1$, a penalty proportional to $(\ell - 1)$ is subtracted. $\lambda$ can be set by cross-validation, an empirical-Bayes prior on lag length, or fixed by a small grid; in all cases, the score remains decomposable. This mirrors how time-series models are evaluated using effective sample sizes (\cite{bib25}). As the cited literature shows, every component of our score is standard. However, to our knowledge, no prior work combines node-specific effective sample sizes with an explicit lag length prior to yield a single, decomposable BN score for edge-specific lags. Each node $v_{j,t}$ contributes a local term given its parents $Pa(v_{j,t})$:

\begin{equation}
S(\mathcal{B}_S)
=\sum_{j=1}^{N}
\Bigg[
  2\log L_j
  \;-\;
  p_j \log n_j
  \;-\;
  \lambda \sum_{(i,\ell)\in Pa_\ell(v_j)} \max\!\bigl(0,\ell-1\bigr)
\Bigg],
\label{eq:score-decomp}
\end{equation}

where $\log L_j$ is the log-likelihood of $v_{j}$ given its lagged parents, $p_j$ is the number of parameters estimated for $v_j$, $n_j$ is the node-specific effective sample size, $ Pa_\ell(v_j)$ is the set of incoming lagged parents to node $j$, and the last term is a lag penalty that discourages unnecessarily long lags.

For $v_j\in V^{\text{cont}}$, we use linear regression; for $v_j\in V^{\text{disc}}$, we use an appropriate Generalised Linear Model (GLM). We focus on binary variables (and therefore use logistic regression) in this study, but the algorithm extends to multiclass and count-type variables without modifying the search. In this regression-based parameterisation, discrete children can have continuous parents, avoiding a standard restriction of conditional Gaussian mixed BNs. Here, we fit each local conditional distribution directly using an appropriate regression model, so a discrete child can depend on continuous parents. Let $g_j(\cdot)$ be the link for node $j$, $x_{j,t}$ the vector of lagged parents at time $t$ (including intercept), and $\theta_j$ the parameters. Then
\[
g_j\big(\mathbb{E}[v_{j,t}\mid Pa(v_{j,t})]\big)=x_{j,t}^\top\theta_j,\quad
\log L_j=\sum_{t=\tau_j}^{T}\log f_j\!\left(v_{j,t}\,\middle|\,x_{j,t},\theta_j\right),
\]
with $\tau_j = 1 + \max_{(i,\ell)\in Pa_\ell(v_j)} \ell$ and $f_j$ the GLM density.

In a BN, the joint likelihood factorises into node-wise conditionals. After unrolling time, each node $v_{j,t}$ has parents amongst $\{v_{i,t-\ell}\}$. Fitting a GLM for $v_{j,t}\mid Pa(v_{j,t})$ and summing the node-wise log-likelihoods yields the joint log-likelihood. The use of (generalised) linear models keeps closed-form likelihoods and a decomposable BIC-style score. Nonlinear extensions are possible without changing the search, but we leave these as extensions.

\subsubsection{Lag adjustment}

For each edge in the network, the algorithm performs a local search over possible lag values to optimise the overall score:

\begin{enumerate}[label=(\roman*)]
  \item For a parent $(i,\ell)\in Pa_\ell(v_{j,t})$, test $\ell'=\ell+1$ (if $\ell<L_{\max}$) and $\ell'=\ell-1$ (if $\ell>1$) by replacing $(i,\ell)$ with $(i,\ell')$ in $Pa_\ell(v_{j,t})$ and recomputing the local score of $v_j$.
  \item If either adjustment improves the score, set the parent to $(i,\ell')$ with the best improvement.
  \item Repeat until no single-step change of $\ell$ improves the score.
\end{enumerate}

\subsubsection{Algorithm pseudocode}

Algorithm~\ref{alg:tabu_varlag_main} and Algorithm~\ref{alg:tabu_varlag_lagtune} provide the algorithm pseudocode below.

\begin{algorithm}[ht]
\caption{Tabu-based Structure Learning with Edge-Specific Lags (Main Loop)}
\label{alg:tabu_varlag_main}
\begin{algorithmic}[1]
\Require Time-series data $D$ over $N$ variables, max lag $L_{\max}$, decomposable score $S(\cdot)$,
Tabu length $L_{\text{tabu}}$, max iterations $I_{\max}$, move set
$\mathcal{M}=\{\textsc{add},\textsc{del},\textsc{rev},\textsc{chg-lag}\}$.
\Ensure Best-scoring graph $G^\star$ with lags in $\{1,\dots,L_{\max}\}$.

\State Initialise $G \gets G_0$; \ $G^\star \gets G$; \ $S^\star \gets S(G)$.
\State Initialise Tabu list $\mathcal{T}\gets\emptyset$.

\For{$t=1$ \textbf{to} $I_{\max}$}
  \State $G_{\text{best}} \gets \emptyset$; \ $S_{\text{best}}\gets -\infty$; \ $m_{\text{best}}\gets \emptyset$.
  \State Construct neighbourhood $\mathcal{N}(G)$ by applying one $m\in\mathcal{M}$ to $G$ that
         respects lag bounds and temporal constraints.

  \For{\textbf{each} move $m \in \mathcal{N}(G)$}
    \State $G_m \gets \textsc{ApplyMove}(G,m)$
    \If{$m \notin \mathcal{T}$ \textbf{or} $S(G_m) > S^\star$}
      \If{$m$ is \textsc{add} or \textsc{rev}}
        \State Initialise the affected/new edge lag at $\ell=1$.
      \EndIf
      \If{$m$ is \textsc{add} or \textsc{rev} or \textsc{chg-lag}}
        \State Let $v$ be the child node whose parent lag changed in $G_m$.
        \State \textsc{GreedyLagTune}$(G_m,v)$ \Comment{Alg.~\ref{alg:tabu_varlag_lagtune}}
      \EndIf
      \State $S_m \gets S(G_m)$.
      \If{$S_m > S_{\text{best}}$}
        \State $G_{\text{best}} \gets G_m$; \ $S_{\text{best}} \gets S_m$; \ $m_{\text{best}} \gets m$.
      \EndIf
    \EndIf
  \EndFor

  \If{$G_{\text{best}}=\emptyset$}
    \State \textbf{break}
  \EndIf

  \State $G \gets G_{\text{best}}$
  \State Add inverse move to $\mathcal{T}$ with tenure $L_{\text{tabu}}$; decrement tenures; remove expired.

  \If{$S(G) > S^\star$}
    \State $G^\star \gets G$; \ $S^\star \gets S(G)$.
  \EndIf
\EndFor
\State \Return $G^\star$
\end{algorithmic}
\end{algorithm}

\begin{algorithm}[ht]
\caption{\textsc{GreedyLagTune}$(G,v)$: Single-step Lag Optimisation for One Child}
\label{alg:tabu_varlag_lagtune}
\begin{algorithmic}[1]
\Require Graph $G$, child node $v$, max lag $L_{\max}$, score $S(\cdot)$.
\Ensure Updated $G$ with locally improved lags on incoming edges to $v$.

\Repeat
  \State $changed \gets \textbf{false}$
  \For{\textbf{each} parent edge $(u,\ell)\in Pa_G(v)$}
    \State $best\_\ell \gets \ell$; \ $bestS \gets S(G)$

    \If{$\ell < L_{\max}$}
      \State Temporarily set lag to $\ell+1$; \ $S_+ \gets S(G)$
      \If{$S_+ > bestS$} $bestS \gets S_+$; $best\_\ell \gets \ell+1$ \EndIf
      \State Restore lag to $\ell$
    \EndIf

    \If{$\ell > 1$}
      \State Temporarily set lag to $\ell-1$; \ $S_- \gets S(G)$
      \If{$S_- > bestS$} $bestS \gets S_-$; $best\_\ell \gets \ell-1$ \EndIf
      \State Restore lag to $\ell$
    \EndIf

    \If{$best\_\ell \neq \ell$}
      \State Set lag on $(u\to v)$ to $best\_\ell$
      \State $changed \gets \textbf{true}$
    \EndIf
  \EndFor
\Until{$changed=\textbf{false}$}
\end{algorithmic}
\end{algorithm}

\subsubsection{Algorithmic guarantees}

The search operates in the DAG space over a time-unrolled graph, with edges constrained to flow from past to future. At every iteration, the current graph remains acyclic: all allowed edges are of the form $v_{i,t-\ell}\rightarrow v_{j,t}$ with $\ell\ge 1$ where it is possible that $i=j$, so time strictly increases along any directed path, which rules out directed cycles (Appendix~\ref{proof_guarantees}). We emphasise that this guarantee applies to the time-unrolled graph. The compact lagged graph may contain directed cycles across variables, corresponding to lagged feedback, without violating acyclicity after unrolling it into a DAG.

The algorithm uses an HC initialisation followed by a Tabu phase. The greedy initialisation terminates at a local optimum with respect to the neighbourhood induced by the allowed moves (add/delete/reverse/change-lag) and the chosen score (Appendix~\ref{proof_guarantees}). The subsequent Tabu phase may traverse non-improving neighbours to escape local optima; accordingly, we record and return the best-scoring graph encountered during the run, denoted $G^\star$.

Because contemporaneous edges are prohibited, every adjacency in an admissible time-unrolled graph connects two different time points and has only one permitted orientation. If two admissible DAGs were Markov equivalent, they would have the same skeleton; temporal order would then orientate every edge in that skeleton identically, implying that the two DAGs were the same graph. Hence, temporal ordering reduces each Markov-equivalence class to at most one admissible representative. The scoring function selects among different skeletons and lag assignments rather than resolving Markov-equivalent orientations.

\subsubsection{Complexity}\label{sec:complex}

Time and space complexity are critical aspects of algorithm design and analysis. Time complexity, as measured by $O$ (big O), is a formal way of describing how an algorithm's running time grows as the input size increases. Big O refers to the order of the function $O(f(n))$, a function of the size of the input $n$, where $f(n)$ is an upper bound on the number of operations the algorithm can perform. An algorithm with a lower time complexity is generally more efficient and can handle larger datasets. Space complexity measures the amount of memory an algorithm requires as a function of the input size.

Although the learnt structure can be viewed as a time-unrolled DAG over $N\times T$ nodes, the algorithm does not materialise the unrolled graph. Instead, it searches in a compact representation over the $N$ original variables, where each directed edge is annotated with a lag $\ell\in\{1,\dots, L_{\max}\}$. A compact structure can contain directed feedback cycles across variables, but it always induces an acyclic time-unrolled graph because all edges point strictly forward in time ($\ell\ge1$).

The score in Eq.~\eqref{eq:score-decomp} is decomposable, so a move that changes the parent set (or lag) of a single child variable $v_j$ only changes that node’s local term. Let $d_j$ be the in-degree of $v_j$ (number of parents in the compact graph), and let $p_j$ be the number of regression parameters (including intercept and any dummy variables from categorical parents). The effective sample size for node $j$ is $n_j = T - \max_{(i,\ell)\in Pa_\ell(v_j)} \ell$.

For a fixed child node $v_j$, the code constructs a design matrix $X\in\mathbb{R}^{n_j\times p_j}$ and response vector $y\in\mathbb{R}^{n_j}$. For continuous child, we use Ordinary Least Squares (OLS), implemented via matrix factorisation ( \texttt{np.linalg.lstsq(X, y)} (\cite{bib26})). When $n_j\ge p_j$, forming such a factorisation costs $O(n_j p_j^2)$. This is because we have $p_j$ orthogonalisation steps, each involving at most $n_j \times p_j$ entries. After factorisation, we must obtain $\hat\beta$. This step depends primarily on $p_j$ because it involves a $p_j \times p_j$ matrix. If these matrices are dense, to solve the linear system, we use each of the $p_j$ columns to remove one variable from the remaining equations, and then update at most $p_j \times p_j$ matrices representing the rest of the system. Therefore, the cost of this step is at most $O(p_j^3)$. Adding these two steps yields the OLS cost, which is $O(n_j p_j^2 + p_j^3)$. The residual computation is $y - X\hat\beta$, costing $O(n_j p_j)$. Since only the dominant cost matters, we obtain the OLS cost, which is $O(n_j p_j^2 + p_j^3)$.

We fit a binary child by Iteratively Reweighted Least Squares (IRLS) with $K$ iterations. In each IRLS iteration, the code performs the following operations:

\begin{enumerate}[label=(\roman*)]
  \item $\eta = X\beta$ costing $O(n_j p_j)$.
  \item Compute $\mu=\sigma(\eta)$, $W=\mu(1-\mu)$, and $z=\eta+(y-\mu)/W$ \quad costing $O(n_j)$.
  \item Form weighted response: $X_w = \operatorname{diag}(\sqrt{W})X$ and $z_w=\sqrt{W}\,z$ \quad costing $O(n_j p_j)$.
  \item Form the weighted cross-products:
        \[
        A = X_w^\top X_w + \text{ridge}\cdot I, \qquad b = X_w^\top z_w.
        \]
        Computing $X_w^\top X_w$ costs $O(n_j p_j^2)$ (a $(p_j\times n_j)$ times $(n_j\times p_j)$ multiply, plus $n_{j}$ additions for the dot product), and $X_w^\top z_w$ costs $O(n_j p_j)$.
  \item As described above for a continuous child, the cost of solving the linear system is $O(p_j^3)$.
\end{enumerate}

All other operations in the code are lower order. Hence, one IRLS iteration costs $O(n_j p_j^2 + p_j^3)$, and with $K$ IRLS iterations until convergence, the IRLS cost for node $j$ is $O\!\big(K(n_j p_j^2 + p_j^3)\big)$.

From this analysis of continuous and binary children, we can calculate the worst-case computational cost ($C_{\text{fit}}$) to fit the local model for one node. During Tabu search, each candidate move changes the parent set or lags for one child node, so we repeatedly refit a single regression (linear or logistic). $C_{\text{fit}}$ is the big-O cost for that refit. We can define
\[
C_{\text{fit}}(T,d_{\max}) \;=\; O\!\big(K\,T\,(d_{\max}+1)^2 + K\,(d_{\max}+1)^3\big),
\]
where $d_{\max}=\max_j d_j$. $C_{\text{fit}}$ depends on $T$ because the effective sample size $n_{j}$ is roughly (and at worst) $T$. $C_{\text{fit}}$ is obtained by substituting the worst-case bounds $n_{j} = O(T)$ and $p_{j} = d_{\max} + 1$ into the logistic cost, which is computationally more expensive than OLS.

The next two theorems characterise the algorithm's time and space complexity. Their proofs can be found in Appendix~\ref{proof_complexity}.

\vspace{1em}
\begin{theorem}[Time Complexity] \label{thm:time_complexity}
Let $H$ be the number of HC iterations and $I$ the number of Tabu iterations. In each iteration, the neighbourhood includes $O(N^2)$ add/remove/reverse move candidates and $O(|E|)$ lag-change candidates, where $|E|$ is the number of edges in the current compact graph. With decomposable scoring, the worst-case time complexity is:
\[
O\!\Bigl((H+I)\,(N^2+|E|)\,C_{\text{fit}}(T,d_{\max})\,L_{\max}\Bigr),
\]
where the factor $L_{\max}$ upper-bounds the number of single-step lag adjustments per affected edge during greedy lag tuning.
\end{theorem}

\vspace{1em}
\begin{theorem}[Space Complexity]\label{thm:space_complexity}
The total space used by the search is
\[
O(NT + |E| + L_{\text{tabu}}),
\]
where $|E|$ is the number of edges in the compact graph and $L_{\text{tabu}}$ is the Tabu list length.
\end{theorem}

\subsection{Theoretical consequences of the single-delay restriction}
\label{sec:single_delay_theory}

As discussed in Section~\ref{sec:tabu-variable}, the proposed method allows different ordered variable pairs to have different lags, but at most one lag per ordered pair. This restriction yields a parsimonious dominant-delay representation, but it also limits recovery when the true data-generating process includes multiple lags for the same ordered pair. The following results quantify this limitation. The proofs are available in Appendix~\ref{proof_delay}.
\par\vspace{\baselineskip}\noindent
\begin{theorem}[Upper bound on structural recovery]
\label{thm:lag_recovery_bound}
Let
\[
E^{\star} \subseteq V\times V\times\{1,\ldots,L_{\max}\}
\]
denote the true set of lag-specific directed edges, and let
\[
P^{\star} = \left\{(i,j):\text{there exists an $\ell$ such that }(i,j,\ell)\in E^\star \right\}
\]
denote the corresponding set of ordered pairs of variables.
\par\vspace{\baselineskip}\noindent
Let $M=|P^{\star}|$ and, for each $(i,j)\in P^\star$, let us define
\[
K_{ij} =\left|\left\{\ell:(i,j,\ell)\in E^\star\right\}\right|.
\]
Thus,
\[
|E^\star|=\sum_{(i,j)\in P^\star}K_{ij}.
\]

For any estimator whose output $\widehat{E}$ contains at most one lag for each
ordered pair,
\[
\operatorname{Recall}\leq\frac{M}{|E^\star|}=\frac{1}{\overline{K}},
\]
where
\[
\overline{K}=\frac{1}{M}\sum_{(i,j)\in P^\star}K_{ij}
\]
\par\vspace{\baselineskip}\noindent
is the average number of true lags per ordered pair. Moreover, the F1 score satisfies
\[
F1\leq\frac{2M}{M+|E^\star|}=\frac{2}{1+\overline{K}}.
\]
\end{theorem}

\begin{corollary}[Mixture of single-lag and multi-lag relationships]
\label{cor:lag_recovery_mixture}
Suppose a proportion $q$ of the true ordered pairs has exactly $K$ true lags, whereas the remaining proportion $1-q$ has one true lag. Then
\[
\overline{K}=(1-q)+qK=1+q(K-1).
\]
It follows that
\[
\operatorname{Recall}\leq\frac{1}{1+q(K-1)}
\]
and
\[
F1\leq\frac{2}{2+q(K-1)}.
\]
\end{corollary}
\par\vspace{\baselineskip}\noindent
For example, if every ordered pair has two true lags, then no single-delay estimator can achieve recall greater than $1/2$, even under perfect ordered-pair recovery. If every ordered pair has three true lags, the corresponding upper bound is $1/3$. These structural bounds do not indicate how much predictive information is lost when the omitted lags are excluded. The following result characterises this loss under logarithmic scoring.
\par\vspace{\baselineskip}\noindent
\begin{theorem}[Irreducible information loss]
\label{thm:information_loss}
Let
\[
Z_t = \left(X_{t-\ell_1},\ldots,X_{t-\ell_K}\right)
\]
denote the vector of true lagged predictors for an outcome $Y_t$, and
let $Z_{-m,t}$ denote all elements of $Z_t$ other than $Z_{m,t}$. Under logarithmic loss, the minimum population loss obtained using all true
lags is
\[
H(Y_t\mid Z_t),
\]
\par\vspace{\baselineskip}\noindent
whereas the minimum population loss obtained using only one lag is
\[
\min_{m\in\{1,\ldots,K\}}
H(Y_t\mid Z_{m,t}).
\]
The irreducible additional loss imposed by the single-delay restriction is
therefore
\[
\Delta_{\log}=\min_{m\in\{1,\ldots,K\}}I\left(Y_t;Z_{-m,t}\mid Z_{m,t}\right)\geq 0.
\]
This loss is zero if and only if there exists a lag $m$ such that
\[
Y_t\perp\!\!\!\perp Z_{-m,t}\mid Z_{m,t}.
\]
\end{theorem}

Theorem~\ref{thm:information_loss} applies to both continuous and discrete outcomes. It shows that the single-delay restriction causes no irreducible predictive loss only when a single lag contains all the information about the outcome available from the remaining lags. If other lags contain additional information conditional on the selected lag, some population predictive loss is unavoidable.

For continuous nodes, we can express this loss explicitly under a linear-Gaussian model.
\par\vspace{\baselineskip}\noindent
\begin{corollary}[Prediction loss under a linear-Gaussian model]
\label{cor:gaussian_prediction_loss}
Suppose that the true model is
\[
Y_t=\boldsymbol{\beta}^{\mathsf T}Z_t+\varepsilon_t,
\]
where
\[
\mathbb{E}[Z_t]=\boldsymbol{0},
\qquad
\operatorname{Cov}(Z_t)=\boldsymbol{\Sigma},
\]
and
\[
\varepsilon_t\sim N(0,\sigma^2)
\]
\par\vspace{\baselineskip}\noindent
is independent of $Z_t$. Assume that $\boldsymbol{\Sigma}$ is positive definite. The population mean-squared prediction error of the correctly specified multi-lag model is

\[
\operatorname{MSE}_{\mathrm{full}}=\sigma^2.
\]
\par\vspace{\baselineskip}\noindent
If only lag $m$ is retained, the population-optimal coefficient is
\[
\alpha_m=\frac{\operatorname{Cov}(Y_t,Z_{m,t})}{\operatorname{Var}(Z_{m,t})}
=
\frac{\boldsymbol{e}_m^{\mathsf T}\boldsymbol{\Sigma}\boldsymbol{\beta}}{\Sigma_{mm}},
\]
where $\boldsymbol{e}_m$ is the $m$th coordinate vector. The minimum single-lag prediction error is
\[
\operatorname{MSE}_{\mathrm{single}}=\sigma^2+\Delta,
\]
where
\[
\Delta =\boldsymbol{\beta}^{\mathsf T}\boldsymbol{\Sigma}\boldsymbol{\beta}
-
\max_{m\in\{1,\ldots,K\}}
\frac{\left(\boldsymbol{e}_m^{\mathsf T}
\boldsymbol{\Sigma}\boldsymbol{\beta}\right)^2}{\Sigma_{mm}}.
\]

The excess prediction error satisfies $\Delta \geq 0$. Furthermore, if $\boldsymbol{\beta}$ contains at least two non-zero components, then $\Delta>0$.
\end{corollary}

These results separate two consequences of the single-delay restriction. Theorem~\ref{thm:lag_recovery_bound} shows that complete structural recovery is impossible when an ordered pair has several true lags. Theorem~\ref{thm:information_loss} and Corollary~\ref{cor:gaussian_prediction_loss} show that predictive loss depends on how much information the omitted lags contain beyond the selected lag.

\subsection{Parallelisation for scalability}\label{sec:parallel}

The variable-lag Tabu runtime (Sec.~\ref{sec:tabu-variable}) can become too slow on real-world datasets. In particular, when the number of variables $N$ is moderate to large, the neighbourhood explored at each iteration is considerable. For example, on the UK COVID-19 policy dataset illustrated later (Sec.~\ref{sec:covid}), the combination of $N=46$, mixed variable types, and repeated local GLM fitting caused the original implementation to run slowly, motivating a set of implementation changes aimed at reducing the elapsed runtime in practice without changing the objective function or search logic.

A separate implementation of the proposed algorithm preserves the same basic move types and score definition, but accelerates the search by parallelising neighbourhood evaluation and improving score bookkeeping. This is achieved by refactoring Eq.~\eqref{eq:score-decomp} so that the objective is maintained as a sum of cached local node scores,

\[
S(G) = \sum_{j=1}^N S_j\!\bigl(Pa_\ell(v_j)\bigr),
\]

where a candidate move modifies the parents or lags of a single child $v_j$ only triggers recomputation of $S_j$. This score bookkeeping reduces the computational cost of candidate evaluation and enables fully independent parallel evaluation across candidates.

The dominant computational cost per iteration is evaluating the candidate moves in the neighbourhood, since each admissible candidate requires fitting an OLS/GLM model to compute the updated log-likelihood term. To reduce elapsed runtime, we generate the candidate move list for the current graph and evaluate those candidates in parallel across CPU cores. After computing all candidate scores, we select the highest-scoring admissible candidate.

The data matrix is placed in shared memory in the parallel version of the algorithm. Then, workers have access to a common read-only backing array rather than receiving separate copies. This reduces the parallel coordination cost. While the parallel version is written to make the same search moves as the original version by default, but faster by evaluating candidates concurrently, it introduces new settings that may cause the parallelised version of the algorithm to behave differently from the non-parallelised version.

Under matched settings, the parallel implementation keeps the same score definition, neighbourhood move types, candidate enumeration order, and move selection logic as the original implementation. By ``matched settings'' we mean matching not only $L_{\max}$ and the lag-penalty coefficient, but also whether the algorithm performs the initial HC before entering the Tabu phase, and how many improvement rounds will be allowed, as well as the Tabu rounds and the length of the Tabu list. In addition, the parallel implementation enabled us to introduce settings that can change the search path. In particular, if we tune lags after a lag-change move and if we re-tune all affected children after reversal. Finally, changing the exposed IRLS controls for binary node fitting can slightly alter local logistic scores and, in turn, candidate selection.

Below, we offer the updated time complexity of the parallel version of the algorithm, the proof can be found in Appendix~\ref{proof_complexity}.
\par\vspace{\baselineskip}\noindent
\begin{theorem}[Parallel complexity]\label{thm:parallel_complexity}
Let $P$ denote the number of worker processes used to evaluate candidate moves. Assuming a balanced workload across workers and a negligible parallel-coordination cost, the worst-case time of the parallel implementation is

\[
O\!\left((H+I)\left[(N^2+|E|)+
\frac{(N^2+|E|)\,C_{\mathrm{fit}}(T,d_{\max})\,L_{\max}}{P}\right]\right)
\]
\par\vspace{\baselineskip}\noindent
When local model fitting dominates bookkeeping, this simplifies to

\[
O\!\left(
\frac{(H+I)(N^2+|E|)\, C_{\mathrm{fit}}(T,d_{\max})\, L_{\max}}{P}
\right)
\]
\par\vspace{\baselineskip}\noindent
Bookkeeping refers to the serial control operations of each iteration: enumerating candidate moves, dispatching them to workers, collecting returned scores, checking tabu admissibility and aspiration (allows a tabu move to be accepted if it improves upon the best score found so far), selecting the best admissible candidate, applying the chosen move, updating the tabu list, and updating the global best solution.
\end{theorem}

\subsection{Summary of the Tabu extension}

The proposed algorithm extends the standard Tabu algorithm to handle time-series causal discovery with variable lags. The main change is that the search is over edges that link past values of one variable to the current value of another. Each edge, therefore, carries its own lag. After a structural change, the algorithm performs a small local adjustment step to check whether increasing or decreasing the lag improves the score. The score itself is modified to reflect the time-series setting, where longer lags leave fewer usable observations for the affected node and are also penalised directly unless they provide a clear improvement in fit. In addition, because all edges are restricted to point from past to future, the time-expanded representation is automatically acyclic.

\section{Results with synthetic data}\label{sec7}

The use of synthetic data is a crucial component in evaluating graphical structure learning algorithms. This is because it provides a controlled environment for testing algorithms under various conditions, thereby enabling an understanding of their strengths and weaknesses within well-defined settings. As in previous research (\cite{bib27}), we assess the proposed algorithm's ability to recover the ground truth graph. Hence, the scoring criteria considered are entirely orientated towards graphical discovery.

\subsection{Structural metrics}

We consider the widely used $F_1$ score (\cite{bib28}), which is based on both recall and precision. Because precision ignores false negatives (missing edges) and recall ignores false positives (spurious ones), they do not independently give a holistic view of the error. The $F_1$ score provides a faithful assessment between the two, since it combines recall and precision into a harmonic mean and penalises extreme imbalances between the two. Eq.~\eqref{eq:f1}  defines the $F_1$ score where precision is noted $P$ and recall is noted $R$:

\begin{equation}
    F_1 \;=\; 2 \;\cdot\; \frac{P \times R}{P + R}
\label{eq:f1}
\end{equation}

$F_{1} \in [0,1]$ where a higher score indicates a more accurate graphical structure recovered, relative to the ground truth.

While the $F_1$ score presents a useful summary of the quality of the learnt structure in recovering true edges, it does not tell exactly how ``far" the learnt graph is from the true one. The Structural Hamming Distance (SHD) (\cite{bib29}) does exactly this. However, unlike the $F_1$ score, it does not account for graph size. The SHD counts the number of steps required to transform the learnt graph into the ground truth graph. This means that a score of 0 indicates a perfect match between the learnt and true graph. As the score increases, we interpret it as an increasing inability of the learnt graph to learn the edges from the ground truth graph. Eq.~\eqref{eq:SHD} defines the SHD score where $n_{\mathrm{add}}, n_{\mathrm{del}}, n_{\mathrm{rev}}$ respectively count the edges that must be added, deleted, or reversed to transform $G_{\text{learnt}}$ into $G_{\text{true}}$.

\begin{equation}
\mathrm{SHD}\bigl(G_{\text{learnt}},\,G_{\text{true}}\bigr)
\;=\;
n_{\mathrm{add}} + n_{\mathrm{del}} + n_{\mathrm{rev}}.
\label{eq:SHD}
\end{equation}

However, the SHD score is known to be biased in favour of sparse graphs, since with each additional edge an algorithm learns, it becomes less likely that it will be a true edge. Therefore, very sparse candidate graphs can look deceptively good with respect to SHD. The Balanced Scoring Function (BSF) (\cite{bib30}) corrects for this graph sparsity bias. Eq.~\eqref{eq:BSF} describes the BSF score where $|E|$ is the size of the set of true edges, $|M|$ is the size of the set of true non-edges, and $TP, TN, FP, FN$ are counts of true positives, true negatives, false positives, and false negatives, respectively.

\begin{equation}
\mathrm{BSF}(G_{\text{learnt}},\,G_{\text{true}})
\;=\;
\frac{1}{2}
\left(
  \frac{\mathrm{TP}}{|E|}
  \;+\;
  \frac{\mathrm{TN}}{|M|}
  \;-\;
  \frac{\mathrm{FP}}{|M|}
  \;-\;
  \frac{\mathrm{FN}}{|E|}
\right)
\label{eq:BSF}
\end{equation}

$BSF \in [-1,1]$, where a higher score represents a more accurate graph. Since BSF is balanced, its normalisation removes bias towards very sparse or very dense graphs, so a score of 0 represents a graph as accurate as an empty or fully connected graph. We include autoregressive self-links (e.g., $v_{i,t-1} \rightarrow v_{i,t}$) as candidate edges when computing SHD and BSF, so these metrics also credit (or penalise) recovery of the autoregressive (AR) structure.

Most of the metrics discussed above were designed for cross-sectional contexts. In the context of time series, however, we are also interested in the time-lagged effect. Therefore, for each recovered edge $(i\!\to\!j)$ with predicted lag $\hat\ell_{ij}$ and true lag $\ell_{ij}$, we report the following metric used in similar studies(\cite{bib31}):
\[
\text{MAE}_\text{lag} \;=\; \frac{1}{|\mathcal{E}_\cap|}\sum_{(i\to j)\in \mathcal{E}_\cap} |\hat\ell_{ij}-\ell_{ij}|,
\]
where $\mathcal{E}_\cap$ is the set of correctly recovered adjacencies.

\subsection{Data generating process}

We generate time series from a fixed, time-invariant DAG with edge-specific lags in $\{1,\dots,L_{\max}\}$. For continuous nodes $Y$, we use a standard linear model:
\[
Y_t = \beta_{0} + \sum_{X\in Pa(Y)} \beta_{X,\ell}\,X_{t-\ell} + \varepsilon^Y_t,\quad \varepsilon^Y_t\sim\mathcal{N}(0,\sigma_Y^2).
\]
For binary nodes, we use a logistic GLM:
\[
\Pr(Y_t=1\mid Pa(Y_t))=\operatorname{logit}^{-1}\!\Big(\beta_{0}+\sum_{X\in Pa(Y)} \beta_{X,\ell}\,X_{t-\ell}\Big).
\]

We iterate over a set of experimental settings in this data-generating process to evaluate the algorithms' learning performance and to understand when and why it succeeds or fails. Each setting represents a different real-world challenge, i.e., a behaviour or pattern we assume to be common in real-world data. Iterating across different settings enables us to isolate the effects of these real-world challenges on the learning process. 

We investigated nine conceptual factors across 10 sweeps, with missingness studied separately under MCAR and MAR mechanisms. This yielded 37 settings and, with 5 independent trials per setting, a total of 185 simulation runs. To be specific, we vary the following settings:

\begin{enumerate}[label=(\roman*)]
    \item The number of unrolled variables, i.e., the number of variables before unrolling them over time. We denote this by $N$.
    \item The sample size (number of time points) $T$.
    \item The true graph’s density, implemented through an edge inclusion probability $p_{edge}$. Higher values increase the expected number of parents, therefore yielding more intertwined networks.
    \item The lag distribution, if they display a long or short memory.
    \item The standard deviation of the residual noise $\sigma_Y$.
    \item Autocorrelation in parents (AR(1) parametrised by $\phi$, where $\phi$ is the parameter of the equation $X_t \approx \phi \times X_{t-1} + noise$).
    \item Proportion of discrete nodes.
    \item The number of latent confounders, which are hidden time series that we do not include as a node in the learnt graph, and affect two observed variables. This hidden common cause can make those two observed series look causally linked even when they are not.
    \item Proportion of missing data. We introduced missingness under both Missing Completely At Random (MCAR), which is missingness unrelated to other variables and to observed and unobserved values of the variable itself, and Missing At Random (MAR), which is missingness related to other variables in the dataset but not to the variable itself, at varying rates. Because the Tabu algorithm requires complete data, we applied a single imputation scheme before generating lagged copies. The imputation was performed per variable using statistics computed from the observed values of that variable over time. For continuous variables, missing entries were imputed with the variable’s sample mean. For binary variables, missing entries were imputed with the variable’s sample mode (ties were broken arbitrarily). This choice intentionally isolates the effect of missingness with this minimal baseline imputation strategy; more sophisticated time-series imputations are left to future work. 
\end{enumerate}

\subsection{Results} \label{sec: results}

As discussed in the previous section, we conducted a series of one-factor-at-a-time simulation sweeps to assess how different properties of the data-generating process affect the performance of the proposed structure-learning algorithm. For each setting, the reported metric value is the mean over the 5 trials. The full sweep level results are shown in Appendix~\ref{full-sweep-results}. Figure~\ref{fig:app-sweeps-f1} reports $F_1$, Figure~\ref{fig:app-sweeps-shd} reports SHD, Figure~\ref{fig:app-sweeps-bsf} reports BSF, and Figure~\ref{fig:app-sweeps-lagmae} reports lag-MAE. Each panel in the figures is titled by the factor being varied. In the discussion below, we refer the reader to the relevant figure and panel title. For sweeps in which the number of variables is fixed, we interpret $F_1$, SHD, BSF, and $\mathrm{MAE}_{\mathrm{lag}}$ together. For the $N$ sweep, however, we avoid comparing SHD across settings because it is sensitive to graph size; instead, we focus on $F_1$ and BSF.

As the sample size increased, structure recovery improved consistently. Mean $F_1$ rose from approximately $0.40$ at $T=500$ to approximately $0.54$ at $T=10{,}000$ (Figure~\ref{fig:app-sweeps-f1}, panel ``Sample size (T)''), while SHD decreased from approximately $7.4$ to $4.6$ (Figure~\ref{fig:app-sweeps-shd}, panel ``Sample size (T)''). BSF also increased steadily (Figure~\ref{fig:app-sweeps-bsf}, panel ``Sample size (T)''), indicating that larger samples improved both edge recovery and the overall balance between correctly identified edges and non-edges.

Increasing the number of variables made structure recovery more challenging. Across the $N$ sweep, BSF declined from approximately $0.70$ at $N=4$ to approximately $0.49$ at $N=24$ (Figure~\ref{fig:app-sweeps-bsf}, panel ``No. variables (N)''), indicating weaker balanced recovery of edges and non-edges as the problem's dimensionality increased. Mean $F_1$ varied more moderately across this sweep (Figure~\ref{fig:app-sweeps-f1}, panel ``No. variables (N)''), suggesting that relative edge recovery performance did not collapse, but the decline in BSF indicates that overall structural recovery became less reliable as the number of candidate relationships increased.

Graph density mainly induced a trade-off between precision and recall. As density increased, precision improved, from approximately $0.21$ at density $0.08$ to approximately $0.52$ at density $0.30$, while recall decreased from approximately $0.82$ to $0.58$. As a result, the mean $F_1$ increased from approximately $0.32$ to $0.55$ (Figure~\ref{fig:app-sweeps-f1}, panel ``Graph density''). However, BSF did not improve in parallel (Figure~\ref{fig:app-sweeps-bsf}, panel ``Graph density'') and was highest in the sparsest setting, indicating that denser graphs were easier to recover in terms of edge overlap but not necessarily in terms of balanced recovery of both edges and non-edges.

Recovery was higher when the data-generating lag distribution aligned with the method’s bias towards shorter delays. The deterioration under the long-memory setting therefore quantifies the cost of prior mismatch and should not be interpreted as evidence that short-memory processes are easier to recover. Mean $F_1$ increased from approximately $0.26$ in the long lag setting to approximately $0.51$ in the short-lag setting (Figure~\ref{fig:app-sweeps-f1}, panel ``Lag distribution''), while BSF rose from approximately $0.32$ to $0.77$ (Figure~\ref{fig:app-sweeps-bsf}, panel ``Lag distribution''). SHD varied less strongly across this sweep (Figure~\ref{fig:app-sweeps-shd}, panel ``Lag distribution'').

Autocorrelation ($\phi$) had the strongest positive effect of any sweep. Performance was lowest at $\phi=0$, with mean $F_1 \approx 0.37$, but increased sharply at $\phi=0.3$ and remained high thereafter, with mean $F_1$ around $0.67$ to $0.69$ for $\phi \in \{0.3,0.6,0.9\}$ (Figure~\ref{fig:app-sweeps-f1}, panel ``Autocorrelation (phi)''). BSF showed the same pattern, rising from approximately $0.56$ at $\phi=0$ to above $0.80$ for moderate and high autocorrelation (Figure~\ref{fig:app-sweeps-bsf}, panel ``Autocorrelation (phi)''). This suggests that, in our setting, temporal persistence made the causal structure easier to detect.

Several sweeps produced mean performance that rose and fell rather than changing consistently with perturbation severity. These patterns should not be interpreted as evidence that noise (Figure~\ref{fig:app-sweeps-f1}, panel ``Noise SD''), missingness (Figure~\ref{fig:app-sweeps-f1}, panels ``MCAR rate'' and ``MAR rate''), proportion of discrete nodes (Figure~\ref{fig:app-sweeps-f1}, panel ``Binary fraction'') or latent confounding (Figure~\ref{fig:app-sweeps-f1}, panel ``No. confounders'') improves structure learning. The perturbations can affect both true and false-edge selection. Increasing the standard deviation of Gaussian disturbances for continuous variables changes both the unexplained variation in continuous child variables and the variability of continuous variables that subsequently act as parents, so it does not simply reduce the signal-to-noise ratio. Mean or mode imputation weakens empirical associations and may accidentally suppress false-positive edges in the relatively dense structures the method returns. Latent common signals can create spurious relationships or obscure true effects, but may occasionally reinforce an existing true relationship in a particular dataset. Because F1 and BSF reflect the balance between several types of error, and because small score changes can lead the discrete Tabu search to select a different graph, modest non-monotonic fluctuations are possible.

Finally, lag recovery was generally accurate whenever an adjacency was correctly recovered. Across most sweeps, $\mathrm{MAE}_{\mathrm{lag}}$ was 0 or very close to 0 (Figure~\ref{fig:app-sweeps-lagmae}), indicating that the algorithm usually selected the correct lag once it had identified the correct edge. Non-zero lag errors appeared only in a small number of settings and remained modest overall, with the largest values occurring in the larger $N$, higher $T$, denser, and moderate autocorrelation settings (Figure~\ref{fig:app-sweeps-lagmae}, panels ``No. variables (N)'', ``Sample size (T)'', ``Graph density'', and ``Autocorrelation (phi)'').

\subsection{Comparison with time-series causal discovery baselines}
\label{sec:comparative-benchmark}

We evaluated all methods using a common benchmark in Python 3.10.21. The serial and parallel implementations of the proposed variable-lag Tabu method were implemented using NumPy 1.23.5 \cite{bib26}. The proposed method used a maximum lag of $L_{\max}=6$, a lag-penalty coefficient $\lambda=0.5$, a Tabu tenure of 10, and a maximum of 100 search iterations. The parallel implementation used four worker processes in the main reviewer-oriented benchmark; the additional standalone multi-lag batch used eight workers.

We implemented two naive score-based controls in the same benchmark. The first restricted every edge to lag 1 (Fixed-1). The second selected a single global lag from $\{1,\ldots,6\}$ by maximising predictive log-likelihood on the final 20\% of the time series, then fitted the structure using that lag (Fixed-CV). These controls distinguish the benefit of selecting an edge-specific lag from that of using either a conventional lag-1 model or a tuned but globally fixed lag.

We implemented PCMCI and LPCMCI using \texttt{Tigramite} 5.2.10.1 \cite{bib11,bib12}, with partial correlation (ParCorr) as the conditional-independence test. We ran PCMCI over lags 1 to 6 with $\alpha_{\mathrm{PC}}=0.05$, an MCI significance level of $0.05$, and Benjamini-Hochberg false-discovery-rate correction. LPCMCI used $\alpha_{\mathrm{PC}}=0.05$ and was restricted to strictly positive lags 1 to 6 to match the target estimand. To ensure computational tractability, its conditioning-set search was bounded using $\texttt{max\_p\_global}=3$, $\texttt{max\_p\_non\_ancestral}=3$, and $\texttt{max\_q\_global}=100$. DYNOTEARS was implemented using \texttt{CausalNex} 0.12.1 \cite{bib32}, with autoregressive order 6, $\lambda_W=\lambda_A=0.1$, a weight threshold of $0.01$, at most 100 optimisation iterations, and an acyclicity tolerance of $10^{-8}$. Because the target estimand contains only strictly positive-lag edges, contemporaneous DYNOTEARS edges were prohibited through its edge-constraint interface. VARLiNGAM was implemented using the \texttt{lingam} 1.12.2 package \cite{bib33}, with autoregressive order 6, pruning enabled, and a coefficient threshold of $0.05$. No ground-truth information was used to select method hyperparameters.

The comparison respected the principal data-type and distributional assumptions of the external methods. PCMCI and DYNOTEARS were evaluated on continuous, causally sufficient data; LPCMCI was evaluated in all continuous settings, including the latent-confounding stress test; and VARLiNGAM was evaluated only under continuous non-Gaussian noise. The proposed method, in its serial and parallel implementations, and the two fixed-lag controls were evaluated across all settings. Consequently, an absent result in Table~\ref{tab:benchmark_f1} denotes a deliberately inapplicable method/setting combination rather than an execution failure. All eligible runs completed successfully.

The principal benchmark comprised eight settings, each containing $N=8$ observed variables, $T=1000$ observations, $L_{\max}=6$, edge-inclusion probability $p_{\mathrm{edge}}=0.15$, continuous coefficient scale $0.8$, binary-logit coefficient scale $1.2$, continuous-noise scale $0.8$, no missing data, and no additional AR(1) persistence ($\phi=0$). Mixed-data settings contained equal numbers of binary and continuous variables and used short-, uniform-, or long-lag distributions. The continuous experiments used the same three lag distributions under Gaussian noise, as well as a Gaussian setting with two latent confounders and a non-Gaussian setting with Laplace noise. We evaluated 20 independently generated datasets in each setting, supplying the same dataset and seed to every eligible method. We report the additional multi-lag model-misspecification experiment separately below.

We separately evaluated recovery of the directed adjacency $(i,j)$ after collapsing over lags and recovery of the exact lagged edge $(i,j,\ell)$. Table~\ref{tab:benchmark_f1} reports both F1 scores as adjacency/exact-lag F1. Exact-lag SHD was calculated as a set symmetric difference, so assigning an incorrect lag incurred one deletion and one addition. Lag-MAE was computed conditional on recovering the correct ordered adjacency, and therefore measures lag assignment rather than overall structural recovery. Paired method differences were summarised using 95\% percentile-bootstrap intervals over the 20 common simulation trials.

\begin{table}[t]
\centering
\caption{Mean adjacency F1/exact-lag F1 over 20 paired simulation trials. A dash indicates that a method was not run because the setting did not match its designated assumptions. The proposed row applies to both implementations, which returned identical structures.}
\label{tab:benchmark_f1}
\small
\setlength{\tabcolsep}{4pt}

\textit{(a) Continuous settings}\par\medskip

\begin{tabular}{lccccc}
\toprule
Method & Uniform & Short & Long & Confounded & Laplace \\
\midrule
Proposed
  & 0.387/0.372 & 0.493/0.493 & 0.288/0.269
  & 0.449/0.433 & 0.376/0.363 \\
Fixed-1
  & 0.190/0.169 & 0.479/0.465 & 0.044/0.013
  & 0.189/0.136 & 0.212/0.163 \\
Fixed-CV
  & 0.281/0.241 & 0.412/0.373 & 0.336/0.313
  & 0.335/0.281 & 0.375/0.349 \\
PCMCI
  & 0.824/0.824 & 0.894/0.888 & 0.815/0.812
  & -- & 0.890/0.890 \\
DYNOTEARS
  & 0.807/0.804 & 0.818/0.818 & 0.805/0.801
  & -- & 0.900/0.900 \\
LPCMCI
  & 0.406/0.383 & 0.413/0.386 & 0.337/0.301
  & 0.442/0.410 & 0.363/0.339 \\
VARLiNGAM
  & -- & -- & -- & -- & 0.864/0.860 \\
\bottomrule
\end{tabular}

\vspace{0.7em}

\textit{(b) Mixed binary--continuous settings}\par\medskip

\begin{tabular}{lccc}
\toprule
Method & Uniform & Short & Long \\
\midrule
Proposed
  & 0.371/0.359 & 0.456/0.451 & 0.259/0.225 \\
Fixed-1
  & 0.122/0.102 & 0.519/0.511 & 0.086/0.011 \\
Fixed-CV
  & 0.258/0.224 & 0.489/0.473 & 0.350/0.319 \\
\bottomrule
\end{tabular}

\end{table}

The external comparisons were intentionally conducted in settings aligned with each comparator's principal assumptions. Within the causally sufficient continuous settings, the proposed method achieved mean adjacency F1 scores between $0.288$ and $0.493$, compared with $0.815$-$0.894$ for PCMCI and $0.805$-$0.900$ for DYNOTEARS. The paired differences between the proposed method and PCMCI ranged from $-0.527$ to $-0.401$, while those against DYNOTEARS ranged from $-0.524$ to $-0.325$; all corresponding bootstrap intervals were below zero. In the Laplace setting, which matches the non-Gaussian assumption of VARLiNGAM, VARLiNGAM achieved adjacency F1 of $0.864$, compared with $0.376$ for the proposed method. The paired difference was $-0.488$, with a 95\% interval of $[-0.552,-0.424]$. These results demonstrate the strength of specialised continuous-data methods when their modelling assumptions are satisfied, while the proposed algorithm provides a common scoring and lag-selection framework for broader scenarios.

The proposed method showed a recall-oriented performance profile. Across the eight principal settings, mean adjacency recall ranged from $0.434$ to $0.897$, whereas precision ranged from $0.198$ to $0.351$. Tabu and hill-climbing searches generally produce denser estimated structures \cite{bib34}, and the present implementation showed the same tendency. It predicted an average of $ 8.7$-$11.5$ adjacencies across settings for ground-truth graphs containing an average of $3.95$-$4.85$. In the causally sufficient continuous settings, this yielded mean adjacency SHD values of $7.7$-$9.4$, compared with $0.8$-$1.05$ for PCMCI and $0.8$-$1.15$ for DYNOTEARS.

This precision-oriented disadvantage is also consistent with differences in model selection in a benchmark favourable to the comparator methods. The proposed method used a mixed-outcome BIC-style score with fixed-strength penalty lag regularisation. The score adds together the model’s fit across all observations. These contributions are normally negative. Therefore, using fewer observations (increasing/adding lags) can make the total score look better simply because fewer negative values are added. Therefore, reinforcing Tabu's tendency to retain a denser set of adjacencies. Thus, the Table~\ref{tab:benchmark_f1} gap primarily reflects weaker false-positive control in these sparse continuous settings, rather than an inability to assign lags to recovered relationships.

SHD should nevertheless be interpreted jointly with F1 and BSF because it can favour very sparse estimates. In the mixed-uniform setting, for example, Fixed-1 achieved a lower SHD than the proposed method ($5.65$ versus $8.15$), but its adjacency F1 and BSF were $0.122$ and $0.076$, respectively, compared with $0.371$ and $0.448$ for the proposed method. The proposed graph therefore recovered considerably more of the true structure despite requiring more edge modifications to reproduce the ground-truth graph exactly.

The bounded LPCMCI comparison produced a more similar recall-oriented profile. Across the five continuous single-lag settings, LPCMCI's adjacency F1 ranged from $0.337$ to $0.442$. The paired adjacency-F1 intervals included zero in the uniform, long-lag, confounded, and Laplace settings; only the continuous-short setting clearly favoured the proposed method, by $0.080$ with a 95\% interval of $[0.048,0.108]$. Specifically for latent confounding, the proposed method and LPCMCI achieved adjacency F1 scores of $0.449$ and $0.442$ and exact-lag F1 scores of $0.433$ and $0.410$, respectively. The paired differences were $0.007$ for adjacency F1 (95\% interval $[-0.035,0.049]$) and $0.023$ for exact-lag F1 (95\% interval $[-0.024,0.072]$), providing no clear evidence of a difference in positive-lag recovery.

The comparison with Fixed-1 provides direct evidence for the value of allowing delays beyond a prespecified lag of 1. The proposed method achieved a higher mean adjacency F1 in seven of the eight principal settings; the only exception was the mixed-short setting ($0.456$ versus $0.519$). Paired intervals favoured the proposed method in six settings and included zero in the mixed-short and continuous-short settings. The largest practical advantages occurred when delays were not concentrated at lag 1: adjacency F1 increased from $0.086$ to $0.259$ in the mixed-long setting and from $0.044$ to $0.288$ in the continuous-long setting.

The comparison with Fixed-CV was more balanced. The proposed method had a higher mean adjacency F1 in five of the eight principal settings, whereas Fixed-CV had a higher mean in three. Paired intervals favoured the proposed method in the mixed-uniform, continuous-uniform, and confounded settings, favoured Fixed-CV in the mixed-long setting, and included zero in the other four settings. These results show that edge-specific lag optimisation improves on a globally selected lag in some settings, but not uniformly.

In the eight single-lag simulation settings, each true directed adjacency was assigned exactly one lag. For the proposed method, mean exact-lag F1 was equal to or no more than $0.034$ lower than mean adjacency F1. Conditional on recovering the correct ordered adjacency, lag-MAE ranged from $0$ to $0.094$ in the short-lag, uniform-lag, confounded, and Laplace-noise settings, but increased to $0.340$ in the continuous long-lag setting and $0.605$ in the mixed long-lag setting. These measures capture different aspects of lag recovery: exact-lag F1 records whether the selected lag and edge are correct, whereas lag-MAE reflects errors in lag selection once the correct edge was recovered. The results therefore indicate that most recovered adjacencies were assigned the correct lag, but that the remaining errors were larger in the long-lag settings. This pattern is consistent with the penalty favouring shorter lags.

\paragraph{Multi-lag model-misspecification experiment.}

We additionally evaluated a data-generating process that violates the proposed method's single-lag-per-adjacency assumption. Results are reported in Table~\ref{tab:benchmark_multilag}. The experiment retained the continuous, causally sufficient, linear-Gaussian specification with $N=8$, $T=1000$, and $L_{\max}=6$, but assigned two distinct uniformly sampled lags to 50\% of the directed adjacencies. The remaining adjacencies retained one lag. At least one multi-lag adjacency was included in every generated graph. Twenty paired datasets were evaluated.

In addition to adjacency and exact-lag F1, three measures isolated performance on true multi-lag adjacencies. Multi-lag edge recall is the proportion of their individual true lagged effects recovered; any-pair recall is the proportion for which at least one correct lag was recovered; and exact-set accuracy requires the predicted lag set for an adjacency to equal the complete true lag set without missing or additional lags.

\begin{table}[t]
\centering
\caption{Mean performance over 20 paired multi-lag simulation trials. ML-edge recall evaluates individual lagged effects on multi-lag adjacencies; any-pair recall requires at least one correct lag; exact-set accuracy requires recovery of the complete lag set.}
\label{tab:benchmark_multilag}
\footnotesize
\setlength{\tabcolsep}{3pt}
\begin{tabular}{lccccc}
\toprule
Method & Adj. F1 & Exact F1 & ML-edge recall & Any-pair recall & Exact-set acc. \\
\midrule
Proposed  & 0.459 & 0.378 & 0.479 & 0.958 & 0.000 \\
Fixed-1   & 0.312 & 0.222 & 0.213 & 0.425 & 0.000 \\
Fixed-CV  & 0.362 & 0.272 & 0.279 & 0.558 & 0.000 \\
PCMCI     & 0.928 & 0.895 & 0.871 & 0.975 & 0.767 \\
DYNOTEARS & 0.883 & 0.872 & 0.869 & 0.975 & 0.763 \\
LPCMCI    & 0.415 & 0.486 & 0.956 & 1.000 & 0.833 \\
\bottomrule
\end{tabular}
\end{table}

The proposed method recovered at least one correct lag for $95.8\%$ of true multi-lag adjacencies and recovered $47.9\%$ of their individual lagged effects. Its exact-set accuracy was zero, as expected, because its current representation can retain only one lag for a given ordered pair. Fixed-1 and Fixed-CV faced analogous representational restrictions and also had zero exact-set accuracy. Relative to Fixed-1, the proposed method improved exact-lag F1 by $0.156$ and adjacency F1 by $0.147$. It also outperformed Fixed-CV by $0.106$ in exact-lag F1 and by $0.097$ in adjacency F1. PCMCI, DYNOTEARS, and LPCMCI achieved exact-set accuracies of $0.767$, $0.763$, and $0.833$, respectively, and exact-lag F1 scores of $0.895$, $0.872$, and $0.486$.

This test clarifies both the benefit and the boundary of the single-lag assumption. It serves as a parsimony restriction that can identify a representative delay for most connected pairs, but it sacrifices the recovery of distributed effects occurring at multiple delays to be useful in settings where sample size $T$ is small relative to the number of potential variables. 

The proposed lag-optimisation routine can be viewed as a modular single-lag subroutine. Because a lag adjustment changes only the parent specification and the local score of the affected child, it can be incorporated either as an additional neighbourhood operator or as a post-search refinement step in conventional score-based causal-ML algorithms. Extending that operator to retain a set of lags per adjacency is a natural direction for future work.

\paragraph{Serial and parallel implementations.}

The serial and four-worker parallel implementations were compared on 180 matched datasets across the nine generated settings. They returned exactly the same edge set across all datasets and identical adjacency, exact-lag, and lag-error metrics. Their final objective values differed by no more than $7.3\times10^{-12}$, consistent with floating-point rounding.

Mean runtime decreased from $41.1$ seconds for the serial implementation to $0.277$ seconds for the accelerated implementation. This is the elapsed wall-clock time required to complete one full structure learning run on one simulated dataset, and the reported values are the means across these 180 runs. Across the 180 matched datasets, the serial implementation required between 68 and 293 times as much computation time as the accelerated implementation. The median runtime ratio (serial over parallel) was 107, while the mean ratio was 136.

\section{Real world application: UK COVID-19 policy decision support}\label{sec:covid}

We test the new algorithm in the real-world setting of the UK COVID-19 pandemic, where policymakers needed evidence (before vaccines/treatments were widely available) about which interventions that reduce population interactions (e.g., reduce mobility and out-of-home activity) are most effective at reducing future infection burden. Petrungaro and Constantinou (\cite{bib35}) frame this as a causal inference problem: learn a causal model from routinely collected daily UK data, then simulate hypothetical interventions (via Pearl’s do-operator (\cite{bib36}) such as lowering mobility indices, and estimate the downstream impact on infection outcomes to support decision-making in future pandemics.

Petrungaro and Constantinou (\cite{bib35}) use an aggregated, publicly available daily dataset spanning from the 30th of January 2020 to the 13th June 2022 with 866 daily observations (861 after processing) and 46 continuous and categorical variables (45 after processing) covering: 

\begin{itemize}
    \item policy (schools, face masks, lockdown severity),
    \item epidemiological context (variant, season),
    \item  mobility/activity proxies (Flights, OpenTable restaurant bookings, Google mobility series, TfL Tube/Bus, Citymapper journeys),
    \item testing volumes and capacities,
    \item pandemic outcomes (new cases, new infections, reinfections, hospital metrics), and
    \item vaccination uptake and deaths.
\end{itemize}

In the original policy-evaluation setup, interventions are framed as manipulating a “population interaction” variable at time $t-1$ and evaluating their effects on infection outcomes at time $t$. In the accompanying implementation, the candidate “interaction” variables correspond to the mobility proxies (e.g., flights, OpenTable/Google/Apple/TfL/Citymapper indices), while infection outcomes include new cases, new infections, and reinfections. Because COVID-19 policy evaluation is contentious regarding effect size but far less so regarding direction (e.g., reducing close-contact interactions should not increase infections), the scenario also naturally supports evaluating whether learnt causal relationships yield directionally plausible intervention effects.

The proposed variable-lag tabu was tested with a maximum lag of $6$, learning a 477-edge lagged structure with 522 free parameters. The fitted log-likelihood and BIC (under our mixed/GLM scoring setup) were $LL = -228{,}775$ and $BIC = -230{,}537$. The variable-lag model sits between very sparse econometric structures and extremely dense score-based BNs, suggesting a different sparsity-fit trade-off under temporal constraints and lag regularisation. Absolute LL/BIC values are not directly comparable across the two studies because of different likelihood models and parameterisations; therefore, comparisons focus on structural properties (density, lag profile, and policy-identifiable links).

The learnt structure favours short delays by design, but not exclusively lag-1 as was the modelling choice of \cite{bib35} (see Figure \ref{fig:lag-dist}). The average lag in this study structure is $2.16$. This is exactly the kind of effect a variable-lag approach is designed to surface. Because the proposed method favours shorter lags, the concentration at lag $1$ should not be interpreted as independent empirical evidence that the underlying system is dominated by one-day effects. Instead, the notable result is that a substantial minority of edges were assigned longer lags despite this preference, indicating that their improvement in fit was sufficient to overcome the imposed parsimony penalty.

These estimates are temporally compatible with possible delays between changes in behaviour, disease transmission, testing, and the reporting of confirmed cases. This interpretation is, however, only suggestive. Because neither the true graph nor the true delays are known in this observational application, epidemiological plausibility does not independently validate the estimated edges or their lags.

\begin{figure}[ht]
\centering
\includegraphics[width=0.8\linewidth]{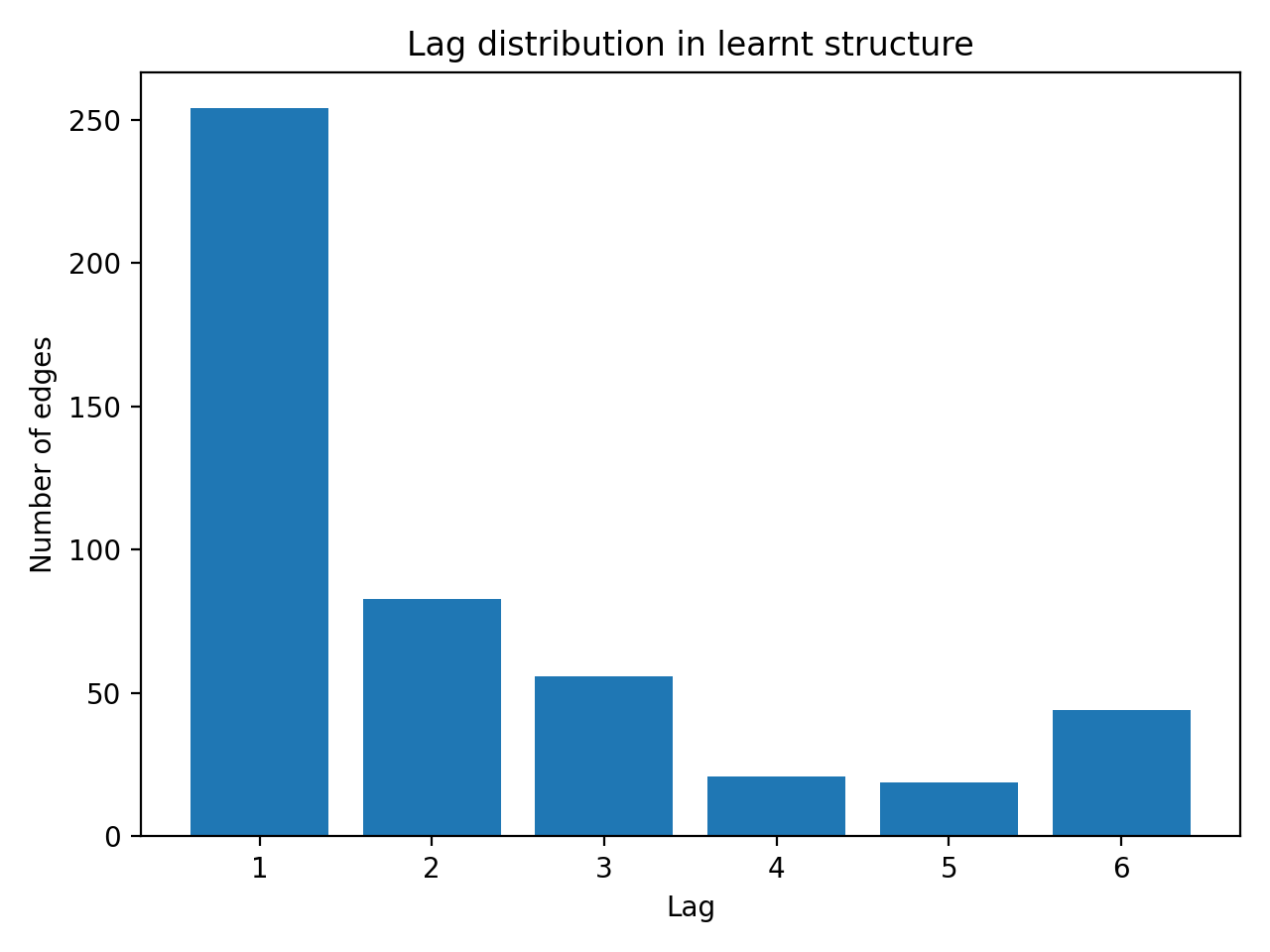}
\caption{Lag distribution of directed edges in the learnt time-series causal graph (maximum lag of $6$). Most dependencies are assigned to lag 1, with a long tail of edges at lags 2 to 6.}
\label{fig:lag-dist}
\end{figure}

\cite{bib35} found that HC and TABU identify 27 mobility $\rightarrow$ outcome effects each, but they also caution that these causal-ML graphs may violate chronological order (as standard TABU and HC are not temporally constrained), limiting direct policy interpretability. To make the policy-effect analysis comparable across studies, we intentionally count only direct lag-1 mobility $\rightarrow$ infection edges. Under this restriction, we identify 3 such effects. More specifically, the direct lag-1 effects are Google parks $\rightarrow$ new cases, Google grocery/pharmacy $\rightarrow$ new cases, and Google grocery/pharmacy $\rightarrow$ new infections. However, only 1 of these 3 direct effects matches the expected direction. Consistent with the density-to-identifiability observation in \cite{bib35}, our sparser, time-respecting model yields fewer immediately identifiable one-step mobility effects than unconstrained dense score-based graphs. However, if we relax this self-imposed restriction, which is likely to be too strict for epidemiological purposes, and allow longer-lag and multi-step directed paths, the proposed algorithm can find directed paths for all 36 mobility $\rightarrow$ outcome effects.

\section{Conclusion}\label{sec:conclusion}

In this study, we introduced a score-based method for learning time-ordered structures from multivariate time series in which different ordered variable pairs can have different delays. The representation retains at most one lag per ordered pair; the selected lag should therefore be interpreted as a dominant or best-supported delay. A Tabu search combines add, delete, reverse, and change-lag moves with a greedy per-child lag-tuning routine. Its decomposable BIC-style objective uses regression-based local likelihoods for continuous and binary children, node-specific effective sample sizes and a lag-length penalty that acts as a parsimony prior favouring recent delays. Because every admissible edge points from the past to the present, each candidate induces an acyclic time-unrolled graph. The greedy initialisation terminates at a local optimum within the induced neighbourhood, and the subsequent Tabu phase can escape local optima while retaining the best-scoring graph encountered during the search.

These design choices maintain the computational benefits of decomposability while promoting parsimonious delay assignments and discouraging reliance on long lags unless they significantly improve model fit. We characterised the cost of this parsimony restriction theoretically and empirically. In the multi-lag experiment, where half of the true ordered pairs had two lags, the method recovered at least one correct lag for 95.8\% of multi-lag pairs but only 47.9\% of their individual lagged effects. A practical advantage of our regression-based mixed-data scoring framework is that, unlike conditional-Gaussian mixed BNs, it does not require discrete children to have only discrete parents. Furthermore, the lag optimiser might be viewed as a modular single-lag subroutine that could be embedded in broader score-based learners.

In simulation, the proposed method recovered both adjacencies and lags, with performance trends that are qualitatively consistent with prior empirical studies of BN structure learning under synthetic perturbations (\cite{bib37}). In particular, the improvements we observed with larger sample sizes $T$ and the deterioration in recovery as the problem becomes harder (i.e., larger $N$) are consistent with the broader empirical literature on structure learning (\cite{bib37}). The proposed method had higher mean adjacency \(F_1\) than Fixed-1 in seven cases, with paired intervals favouring it in six, and the largest gains occurring under long delays. Results against the validation-selected Fixed-CV method were setting-dependent. In continuous settings matched to their assumptions, PCMCI, DYNOTEARS and VARLiNGAM were substantially more accurate, whereas bounded LPCMCI was broadly comparable for positive-lag recovery in single-lag settings. These differences stemmed mainly from adjacency selection; i.e., the proposed method retained denser graphs and therefore had lower precision, while its exact-lag (conditional lag error) \(F_1\) remained close to its adjacency \(F_1\). PCMCI and DYNOTEARS combined linear-Gaussian modelling assumptions that closely matched the simulation design with explicit false-discovery or sparsity controls. Improving sparsity calibration and false-positive control remains an important direction for further work.

The parallel implementation preserved the serial structures and recovery metrics across simulations while reducing mean runtime from $41.1$ to $0.277$ seconds. This improvement reflects local-score caching, shared-memory data access and parallel neighbourhood evaluation. We also demonstrated feasibility in a real-world policy setting: daily UK COVID-19 data spanning 30 January 2020 to 13 June 2022, comprising 46 observed variables. When restricting to directly comparable direct lag-1 mobility $\rightarrow$ infection-outcome edges, the proposed time-ordered Tabu algorithm identified fewer one-step policy-relevant links than dense, unconstrained score-based graphs reported in \cite{bib35}. By removing the restriction to direct lag-1 effects, that is, effects represented by edges of the form $X_{t-1}\rightarrow Y_t$, and instead allowing longer lags, we found that the new algorithm can identify all effects proposed by \cite{bib35}.

Several limitations suggest clear avenues for future research. First, our local models are linear or GLMs; incorporating nonlinear models is possible within the same search framework, though at increased computational cost. Second, we employed simple imputation to isolate the effects of missingness, but time-series-aware imputation methods could enhance robustness under missingness. Finally, the current framework assumes no contemporaneous edges and causal sufficiency; extending the approach to address latent confounding and to allow within-time-slice structure represents an important direction for future work.

\backmatter

\section*{Declarations}

\subsection{Data availability and access}

The data and code used in this study have been submitted alongside the paper. Code, data, replication materials, and additional documentation are available in the accompanying GitHub repository:

\url{https://github.com/br1pa/variable-lag-tabu}

\subsection{Competing interests}

The authors declare that they have no known competing financial interests or personal relationships that could have appeared to influence the work reported in this paper.

\subsection{Ethical and informed consent for data used}

This article does not contain any studies with human participants or animals performed by any of the authors.

\subsection{Contributions}

Bruno Petrungaro: Conceptualisation, Methodology, Software, Analysis, Writing-original draft preparation, review and editing. 
Anthony C. Constantinou: Writing - Review \& Editing, Supervision.

\begin{appendices}

\section{Algorithmic guarantees proofs}\label{proof_guarantees}

\begin{theorem}[Validity]
At any iteration of the algorithm, the time-unrolled graph induced by the current structure is a DAG.
\end{theorem}

\begin{proof}
Every edge in the compact representation has the form $(u,\ell)\to v$ with $\ell\ge 1$, which corresponds in the time-unrolled graph to edges $u_{t-\ell}\to v_t$ for all valid $t$. Along any directed edge, the time index increases by at least one step (from $t-\ell$ to $t$). A directed cycle would require returning to a past time index, which is impossible under strictly increasing time. Hence, directed cycles cannot occur, and the unrolled graph is acyclic at every iteration.
\end{proof}

\begin{theorem}[Local Optimality of the Greedy Initialisation]
Let $G_{\mathrm{hc}}$ be the graph returned by the HC initialisation. Then $G_{\mathrm{hc}}$ is a local optimum with respect to the neighbourhood induced by the allowed moves (add/delete/reverse/change-lag) under the chosen score $S(\cdot)$.
\end{theorem}

\begin{proof}
Let $N(G)$ denote the set of graphs reachable from $G$ by one allowed move (add, delete, reverse, or a lag change) that respects the lag bounds and time-order constraints. The greedy initialisation updates the current graph only when it finds a neighbour $G'\in N(G)$ such that $S(G')>S(G)$, and it terminates exactly when no such improving neighbour exists. Therefore, at termination for $G_{\mathrm{hc}}$ we have $\nexists\,G'\in N(G_{\mathrm{hc}})$ with $S(G')>S(G_{\mathrm{hc}})$, i.e., $\forall\,G'\in N(G_{\mathrm{hc}}): S(G_{\mathrm{hc}})\ge S(G')$.
\end{proof}

\section{Complexity proofs}\label{proof_complexity}

The proof of Theorem~\ref{thm:time_complexity} (Time Complexity) is presented below.

\begin{proof}
Each search iteration evaluates a set of candidate moves (add/remove/reverse) over ordered pairs $(u,v)$. These moves potentially exist for any choice of $(u,v)$. There are $N$ choices of $u$ and $N$ choices of $v$, giving $O(N^2)$ candidates. A lag-change move can only be applied to an edge that already exists in the current compact graph. If the current graph has $|E|$ edges, then there are at most $|E|$ change-lag options, so $O(|E|)$. Therefore, the total number of candidate moves per iteration is: $O(N^2 + |E|)$. Because the score is decomposable, evaluating a move requires re-scoring only the child node whose parent set or lag changed, at a cost of at most $C_{\text{fit}}(T,d_{\max})$. See Sec.~\ref{sec:complex} for a detailed explanation. When lag tuning is applied, each incoming edge to the affected child can be adjusted by $\pm1$ and accepted repeatedly, and in the worst case, an edge’s lag can move across the entire range $1,\dots, L_{\max}$, yielding an $O(L_{\max})$ factor. Multiplying these factors over the $(H+I)$ iterations gives the stated bound.
\end{proof}

The proof of Theorem~\ref{thm:space_complexity} (Space Complexity) is presented below.

\begin{proof}
The algorithm stores the dataset $D$ of size $N\times T$, the compact parent lists representing $|E|$ lagged edges, and a Tabu list of bounded size $L_{\text{tabu}}$. The algorithm only stores the compact lagged graph, not the time-unrolled structure. Summing these terms yields the bound.
\end{proof}

The proof of Theorem~\ref{thm:parallel_complexity} (Parallel Complexity) is presented below.

\begin{proof}

Since the neighbourhood contains $O(N^2)$ add/remove/reverse moves and $O(|E|)$ lag-change moves, we have $O(N^2+|E|)$ as the potential cost of the moves. The dominant serial bookkeeping cost consists of generating the candidate list and scanning the returned candidate scores to select the best admissible move. Each is a single pass over at most $N^2+|E|$ candidates, so together they contribute $O(N^2+|E|)$ per iteration. The remaining per-iteration bookkeeping updates are of lower order because they occur only once after the best candidate has already been chosen, so they do not require examining $N^2+|E|$ candidates.

By decomposability, each add, remove, lag-change, or reverse move requires recomputing the local node scores, each costing at most $C_{\mathrm{fit}}(T,d_{\max})$. When greedy lag tuning is invoked, the same worst-case $O(L_{\max})$ factor as in Theorem 1 applies. Distributing the $N^2+|E|$ candidates across $P$ workers gives $O((N^2+|E|)/P)$ candidates per worker under balanced load, yielding parallel scoring cost

\[
O\!\left(
\frac{(N^2+|E|)\, C_{\mathrm{fit}}(T,d_{\max})\, L_{\max}}{P}
\right).
\]

Considering $(H+I)$ iterations and combining the above conclusions, we get

\[
O\!\left((H+I)\left[(N^2+|E|)+
\frac{(N^2+|E|)\,C_{\mathrm{fit}}(T,d_{\max})\,L_{\max}}{P}\right]\right).
\]

Factoring out $(N^2+|E|)$ gives

\[
O\!\left((H+I)(N^2+|E|)\left[1+\frac{C_{\mathrm{fit}}(T,d_{\max})L_{\max}}{P}\right]\right).
\]

In the regime where candidate scoring dominates bookkeeping, 

\[
\frac{C_{\mathrm{fit}}(T,d_{\max})L_{\max}}{P} \gg 1,
\]

the latter term dominates, so the additive constant $1$ is negligible, yielding

\[
O\!\left(
\frac{(H+I)(N^2+|E|)\,C_{\mathrm{fit}}(T,d_{\max})\,L_{\max}}{P}
\right).
\]

\end{proof}

\section{Theoretical consequences of the single-delay restriction proofs}\label{proof_delay}

The proof of Theorem~\ref{thm:lag_recovery_bound} (Upper bound on structural recovery) is presented below.

\begin{proof}
Let
\[
TP = |\widehat{E}\cap E^\star|
\]
denote the number of correctly recovered edges. Because $\widehat{E}$ contains at most one lag for each ordered pair, each true ordered pair can contribute at most one true positive. Consequently,
\[
TP\leq M.
\]
Recall is therefore bounded by
\[
\operatorname{Recall}
=
\frac{TP}{TP + FN}
=
\frac{TP}{|E^\star|}
\leq
\frac{M}{|E^\star|}
=
\frac{1}{\overline{K}}.
\]

For two finite edge sets, the F1 score can be written as
\[
F1
=
\frac{2TP}{2TP + FP + FN}
=
\frac{2TP}{|\widehat{E}|+|E^\star|}.
\]
Every false positive increases the size of $\widehat{E}$, and therefore the denominator. However, it does not increase the numerator. Therefore, for a given number of true positives, this quantity is maximised when there are no false-positive edges, so that $|\widehat{E}|=TP$. It is then increasing in $TP$. Using the upper bound $TP\leq M$ gives
\[
F1
\leq
\frac{2M}{M+|E^\star|}
=
\frac{2}{1+\overline{K}}.
\]
\end{proof}

The proof of Theorem~\ref{thm:information_loss} (Irreducible information loss) is presented below.

\begin{proof}
Under logarithmic loss, the optimal predictive distribution based on a set of predictors is the true conditional distribution given those predictors. Consequently, the minimum expected logarithmic loss using all components of $Z_t$ is
\[
H(Y_t\mid Z_t),
\]
and the minimum loss using only $Z_{m,t}$ is
\[
H(Y_t\mid Z_{m,t}).
\]

By the definition of conditional mutual information,
\[
I\left(Y_t;Z_{-m,t}\mid Z_{m,t}\right)=H(Y_t\mid Z_{m,t})-H(Y_t\mid Z_{m,t},Z_{-m,t}).
\]
Because $(Z_{m,t},Z_{-m,t})=Z_t$, it follows that
\[
H(Y_t\mid Z_{m,t})-H(Y_t\mid Z_t)=I\left(Y_t;Z_{-m,t}\mid Z_{m,t}\right).
\]
Minimising over $m$ gives
\[
\Delta_{\log}=\min_mI\left(Y_t;Z_{-m,t}\mid Z_{m,t}\right).
\]
Conditional mutual information is non-negative; therefore, $\Delta_{\log}\geq0$. It is equal to zero only when there exists an $m$ for which
\[
Y_t\perp\!\!\!\perp Z_{-m,t}\mid Z_{m,t}.
\]
\end{proof}

The proof for Corollary~\ref{cor:gaussian_prediction_loss} (Prediction loss under a linear-Gaussian model) is presented below.

\begin{proof}
For a fixed lag $m$, consider a single-lag predictor of the form $\alpha Z_{m,t}$. Its population mean-squared error is
\[
\begin{aligned}
R_m(\alpha)
&=
\mathbb{E}\left[\left(Y_t-\alpha Z_{m,t}\right)^2\right] \\
&=
\mathbb{E}\left[\left(\boldsymbol{\beta}^{\mathsf T}Z_t-\alpha Z_{m,t}+\varepsilon_t\right)^2
\right].
\end{aligned}
\]
The vector $\boldsymbol{e}_m$ is the $m$-th coordinate vector. Therefore, $\boldsymbol{e}^{T}_{m}Z_t = Z_{m,t}$. Using the independence of $\varepsilon_t$ and $Z_t$ gives
\[
R_m(\alpha)=\sigma^2
+
\left(\boldsymbol{\beta}-\alpha\boldsymbol{e}_m\right)^{\mathsf T}\boldsymbol{\Sigma}
\left(\boldsymbol{\beta}-\alpha\boldsymbol{e}_m\right).
\]

Differentiating with respect to $\alpha$ and setting the derivative equal to zero gives
\[
\alpha_m=\frac{\boldsymbol{e}_m^{\mathsf T}\boldsymbol{\Sigma}\boldsymbol{\beta}}{\Sigma_{mm}}.
\]

Substituting $\alpha_m$ into $R_m(\alpha)$ yields

\[
R_m(\alpha_m)=
\sigma^2
+
\boldsymbol{\beta}^{\mathsf T}\boldsymbol{\Sigma}\boldsymbol{\beta}
-
\frac{\left(\boldsymbol{e}_m^{\mathsf T}\boldsymbol{\Sigma}\boldsymbol{\beta}\right)^2}{\Sigma_{mm}}.
\]

Selecting the best single lag therefore gives

\[
\operatorname{MSE}_{\mathrm{single}}
=
\min_m R_m(\alpha_m)
=
\sigma^2+\Delta.
\]

Because $\boldsymbol{\Sigma}$ is positive definite,
\[
\left(\boldsymbol{\beta}-\alpha\boldsymbol{e}_m\right)^{\mathsf T}\boldsymbol{\Sigma}\left(
\boldsymbol{\beta}-\alpha\boldsymbol{e}_m\right)\geq0
\]
for every $\alpha$ and $m$, establishing $\Delta\geq0$. Equality can hold only
if
\[
\boldsymbol{\beta}=\alpha\boldsymbol{e}_m
\]
for some $m$, meaning that at most one lag has a non-zero coefficient. Thus, if at least two elements of $\boldsymbol{\beta}$ are non-zero, then $\Delta>0$.
\end{proof}

\section{Full sweep-level simulation results}\label{full-sweep-results}

This appendix reports the full sweep-level results for the one factor at a time simulation experiments discussed in Section~\ref{sec: results}.

\begin{figure}[ht]
    \centering
    \includegraphics[width=\textwidth]{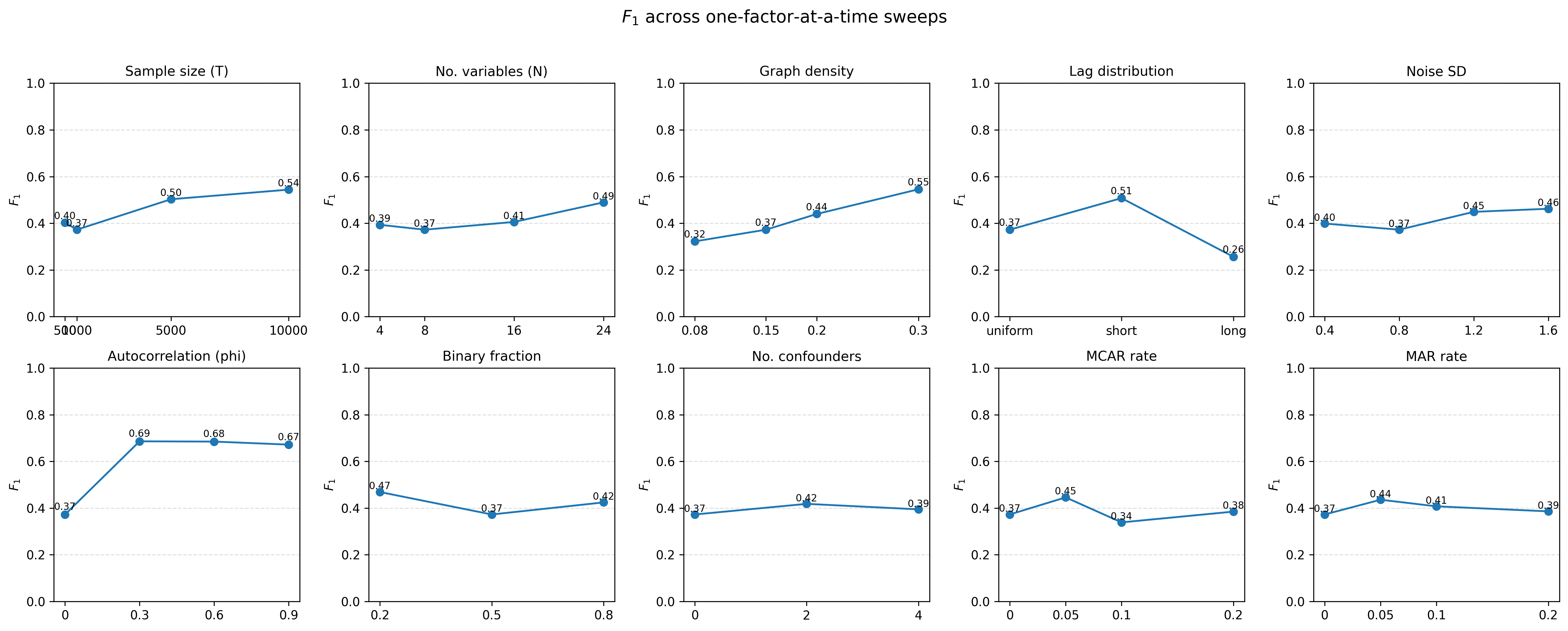}
    \caption{Sweep-level $F_1$ results across the 10 one factor at a time simulation sweeps. Each panel varies one factor while holding the others at their baseline values. Points show the mean performance over 5 trials for each setting.}
    \label{fig:app-sweeps-f1}
\end{figure}

\begin{figure}[ht]
    \centering
    \includegraphics[width=\textwidth]{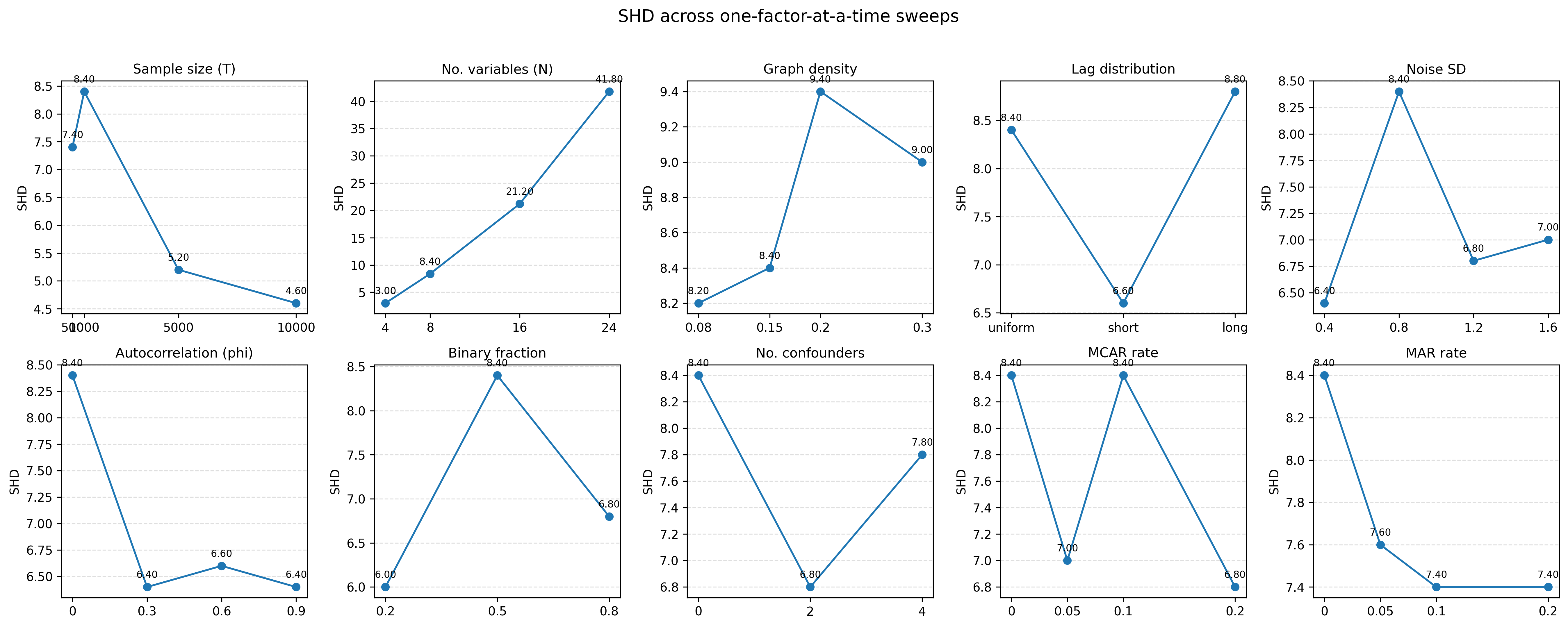}
    \caption{Sweep-level SHD results across the 10 one factor at a time simulation sweeps. Lower values indicate better structural recovery. Points show the mean performance over 5 trials for each setting.}
    \label{fig:app-sweeps-shd}
\end{figure}

\begin{figure}[ht]
    \centering
    \includegraphics[width=\textwidth]{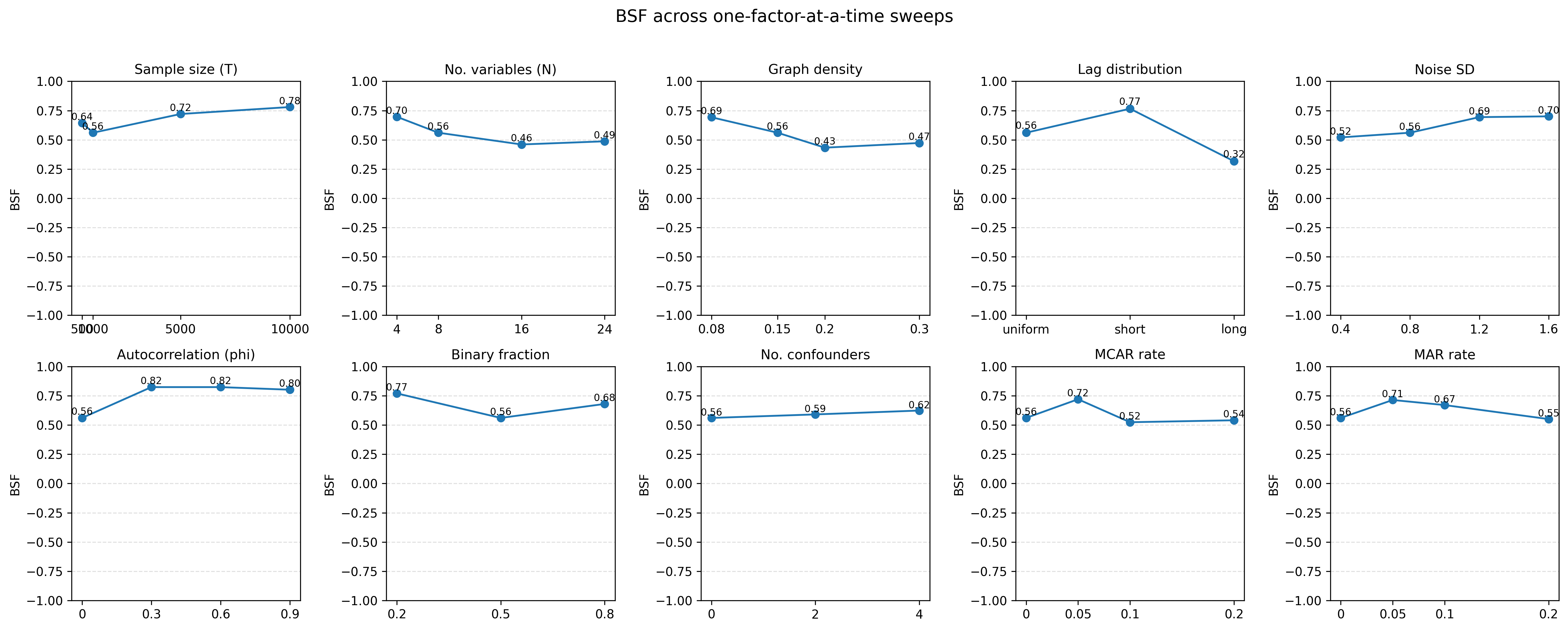}
    \caption{Sweep-level BSF results across the 10 one factor at a time simulation sweeps. Higher values indicate better balanced recovery of edges and non-edges. Points show the mean performance over 5 trials for each setting.}
    \label{fig:app-sweeps-bsf}
\end{figure}

\begin{figure}[ht]
    \centering
    \includegraphics[width=\textwidth]{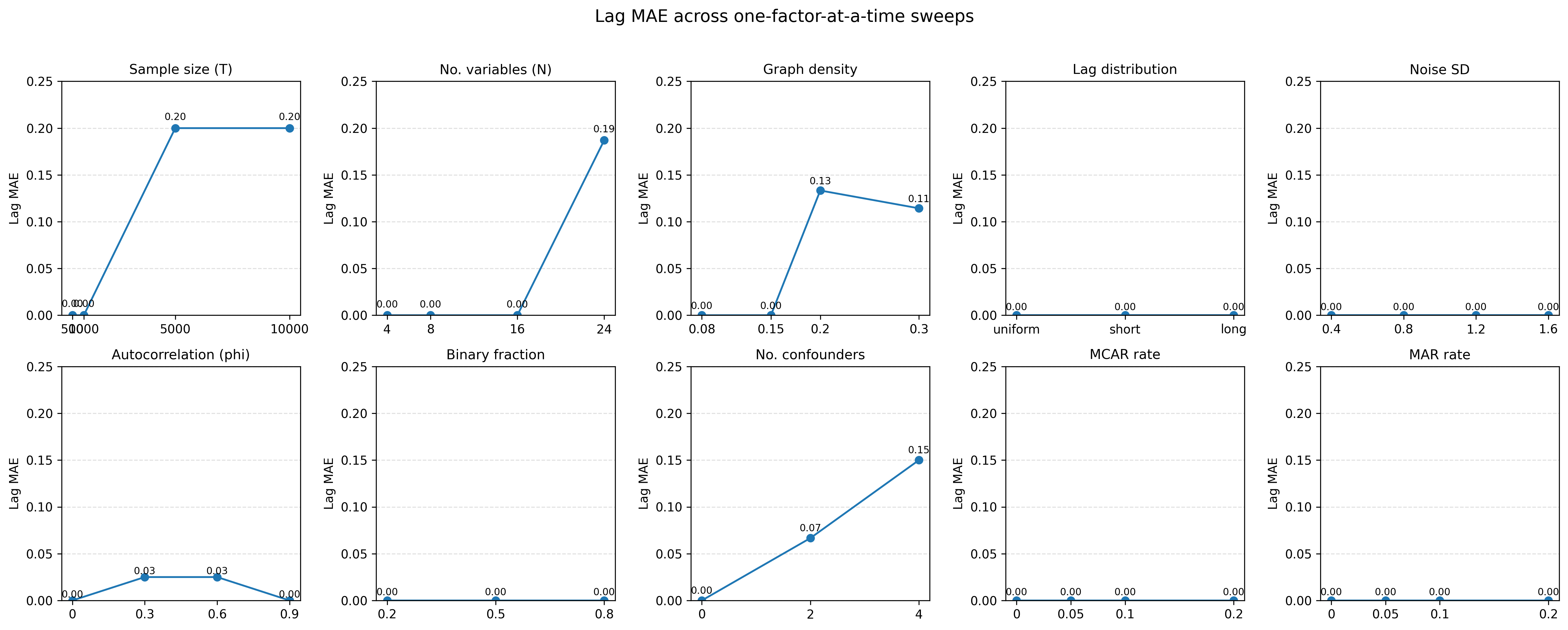}
    \caption{Sweep-level lag-MAE results across the 10 one factor at a time simulation sweeps. Lower values indicate more accurate lag recovery on correctly recovered adjacencies. Points show the mean performance over 5 trials for each setting.}
    \label{fig:app-sweeps-lagmae}
\end{figure}

\end{appendices}

\clearpage
\bibliography{sn-bibliography}

\end{document}